%% file: arXiv_submission/arXiv_sp_2.tex
\documentclass[11pt, letterpaper]{article}
\usepackage[margin=1in]{geometry}
\usepackage[utf8]{inputenc}
\usepackage{url}
\usepackage{amsfonts}  
\usepackage{amsmath}
\usepackage{xcolor}         
\usepackage{mathtools}
\usepackage{dsfont}
\usepackage{amsthm}
\usepackage{multirow} 
\usepackage{natbib,hyperref}
\usepackage{graphicx}

\usepackage{cleveref}
\usepackage[dvipsnames]{xcolor}
\usepackage{tikz}
\usepackage{booktabs}
\usepackage[ruled,vlined, linesnumbered]{algorithm2e}
\usepackage{cleveref}

\usepackage{authblk}

\newcommand{\ale}[1]{\textcolor{Apricot}{~Ale:~#1}}
\newcommand{\aj}[1]
{\textcolor{blue}{~Adel:~#1}}
\newcommand{\ka}[1]
{\textcolor{red}{~Kareem:~#1}}

\usepackage{tabularx}
\usepackage{xcolor}
\usepackage{caption}
\newcommand{\match}[1]{\textcolor{blue!80!black}{\textbf{#1}}}

\newcommand{\cM}{\ensuremath{\mathcal{M}}}

\newcommand{\cF}{\ensuremath{\mathcal{F}}}

\newcommand{\cX}{\ensuremath{\mathcal{X}}}

\newcommand{\Prob}{\mathrm{Pr}}
\newcommand{\E}{\mathrm{E}}
\newcommand{\de}{\mathrm{d}}
\newcommand{\pval}{p_{\rm val}}

\newtheorem{lemma}{Lemma}
\newtheorem{theorem}{Theorem}

\newtheorem{proposition}{Proposition}
\newtheorem{remark}{Remark}

\usetikzlibrary{positioning, arrows.meta,shadows}

\title{Phantoms and Disclosures: A Statistical Framework for Auditing Privacy in Synthetic Data}

\author[1]{Kareem Amin}
\author[1]{Rudrajit Das}
\author[1]{Alessandro Epasto}
\author[1,2]{Adel Javanmard}
\author[1]{Dennis Kraft}
\author[1]{M\'onica Ribero}
\author[1]{Sergei Vassilvitskii\thanks{The authors' names are listed in alphabetical order.}}

\affil[1]{Google}
\affil[2]{University of Southern California}
 \date{}

\begin{document}

\maketitle
\begin{abstract}
The rapid adoption of generative AI and Large Language Models (LLMs) has spurred interest in synthetic data as a privacy-preserving alternative to sensitive real-world datasets. However, generating high-utility synthetic data often carries the risk of memorizing and regurgitating private information from the training corpus. In this work, we present a customizable empirical auditing framework designed to detect and explain such data disclosures. Our framework introduces a mechanism to distinguish between \textit{true disclosures}---where the system directly reproduces a user's information---and \textit{phantom disclosures}---where the system incidentally generates a user's data. By partitioning input data into training and holdout sets and applying rigorous statistical hypothesis testing, we determine if observed disclosures are consistent with strict privacy baselines, such as zero-learning or specific Differential Privacy (DP) bounds. Crucially, this approach requires no model access, no canary insertion, and no reference model training — only the synthetic output and a held-out control set. We demonstrate that this framework effectively functions as a membership inference attack, providing empirical lower bounds on privacy leakage that are tighter than prior data-based auditing methods. 
Our approach is model-agnostic, applies to any synthetic data generation mechanism, and requires orders of magnitude fewer computational resources than shadow-model or canary-based alternatives.
\end{abstract}


%

\section{Introduction}
The rise of generative AI has created a great deal of interest in the study of synthetic, machine-generated, data. Large language models (LLMs) can be prompted to create synthetic documents resembling real data from multiple, often sensitive, domains such as emails, court documents, health records, or search queries. 

As the capabilities of these models continue to increase, AI-generated datasets are becoming a promising alternative to real human data, specifically, as a \emph{privacy-preserving measure}~\cite{jordon2022synthetic}. Consider a hospital service that wishes to study treatment outcomes by sharing its data with external researchers. Rather than contend with the privacy and compliance risks associated with releasing its raw patient records, the hospital might release a \emph{synthetic data} analog, produced with the assistance of an AI model. 

The details behind synthetic data generation can vary greatly across deployments due to the technical capabilities of the data-owner and the requirements of the users of the data. For example, prompting an LLM to produce synthetic data that matches the \emph{schema} of hospital records may be good-enough to stress test a data ingestion pipeline. However, 
the same data is less useful to a researcher interested in the  relationship between treatments and outcomes. For this application, the researcher requires that the synthetic data capture the statistical relationships within the private dataset. 

As a consequence, there are many techniques for generating synthetic data, many of which make use of signal in the sensitive private dataset in some way. From the viewpoint of data privacy, this creates a dilemma. On the one hand, synthetic data can be viewed as a safer alternative to real data in almost any context where real data is used. On the other hand, generating \emph{high-quality} synthetic data often requires capturing precisely the signals that make private data sensitive. As a result, good synthetic data can appear as sensitive as the corpus used to generate it. 

A non-exhaustive list of synthetic data generation techniques includes: prompting an LLM to produce data in some desired format, prompting an LLM to rewrite private data records that have been placed in-context, or fine-tuning an LLM to produce responses similar to the private corpus. There are also several analogs to these approaches that provide formal differential privacy (DP) guarantees~\cite{dwork2006calibrating} such as DP inference, DP fine-tuning, DP tabular data generation,  and private evolution (see~\cite{ponomareva2025dp} for a review). 

To complicate matters further, the privacy properties of the system may depend on a variety of implementation decisions that are not captured in any formal guarantees. For example, scrubbing personally identifiable information (PII) from the source data can be materially beneficial to privacy. Similarly, the choice of privacy parameters in a DP system can often be so large as to render the formal guarantees vaccuous. However, many real deployments (most famously the US Census~\cite{abowd2018us}) argue that there are tangible privacy benefits to these algorithms nonetheless. 

Consequently, there is a critical need for rigorous auditing frameworks to evaluate the privacy guarantees of synthetic data deployments. To remain practical, such an audit must be fundamentally decoupled from—and robust to—the specific algorithmic techniques used to generate the data.

Motivated by these considerations, we adopt the perspective of an internal privacy auditor seeking to quantify leakage risks prior to releasing synthetic data. Our approach centers on disclosure detection within the generated dataset, where a disclosure is defined as any feature or information that occurs rarely in the source corpus yet resurfaces in the synthetic output. By focusing on disclosures, our framework relies only on looking at the data, and offers several distinct advantages. 
 
{\bf Black-box auditing.} Given the myriad ways in which synthetic data might be generated, our system makes no assumptions about the nature of the data or the synthetic data generation procedure. The synthetic data does not need to have been generated by fine-tuning and prompting a model~\cite{yue2023synthetic,mattern2022differentially,kurakin2023harnessing}, but could be the result of any synthetic data generation procedure such as inference~\cite{amin2024private,tang2024privacy} or  evolution~\cite{lin2023differentially}. This contrasts with prior work on model based auditing~\cite{shokri2017membership,carlini2022membership,ye2022enhanced} which can only audit synthetic data generated through a model. In order to conduct the audit, the auditor needs access only to the synthetic data, the private data that was synthesized, and a held-out dataset that was not used for synthesis. 

{\bf Turn-key deployment.} Prior black-box approaches~\cite{meeus2025canary} require designing domain-specific canaries to attack the system and training additional reference models~\cite{zarifzadeh2023low}. Such systems require privacy experts with domain knowledge of the data to ensure adequate canaries are inserted into the system. Our system avoids both the computational bottleneck of fine-tuning large models, and the human-in-the-loop dependency created by canary design. 

{\bf Flexible disclosure detection.} Our system produces witnesses of privacy leakage, called disclosures. Identifying these disclosures is important as it allows the auditor and decision makers to reason concretely about the risks in the synthetic data release. Moreover, the system provides the auditor with a predefined list of disclosure \emph{classes} to choose from. This list is flexible and can be expanded as users of the system identify new types of risks that they wish to measure. 

{\bf Phantom disclosure detection.} We identify and formalize a pervasive but previously overlooked phenomenon: a large fraction of apparent privacy violations in synthetic data are what we call phantoms — disclosures that would occur with equal probability regardless of whether the user's data was included in training. {\color{black}{Existing $k$-extraction methods~\cite{carlini2021extracting}, for example, classify all infrequent strings as memorized and therefore fail to distinguish genuine memorization from these phantom disclosures.}} In our experiments, phantoms account for over 35\% of detected disclosures (e.g., 271 of 763 PII matches in the Finance dataset). This finding has immediate practical implications: naïve disclosure counts dramatically overstate privacy risk. Our system specifically measures and grounds the phantom disclosure rate using holdout data. 

{\bf Formal privacy leakage.} Beyond simple detection, we also use the set of detected disclosures to provide formal measurements of privacy leakage. Our framework provides two levels of analysis: a zero-learning test that detects any dependence between synthetic outputs and source data, and a DP-bounded learning test that assesses whether observed disclosures are consistent with $\varepsilon$-differential privacy. Both tests provide formal type I error control (Theorems~\ref{thm:DP-valid} and \ref{thm:mia_main}) and yield empirical lower bounds on the privacy parameter $\varepsilon$ (Theorem~\ref{thm:eps-LB}). The framework is analogous to a controlled experiment in which the holdout set serves as the control group and the training set as the treatment group. We also construct membership inference attacks using just disclosures, providing yet another formal metric of privacy leakage (the AUC of the attack). Unlike prior MIA approaches, our attack requires no model access, no canary insertion, and no shadow-model training — only the released synthetic data and a held-out sample from the same distribution.

As with any statistical testing (and any empirical privacy testing),  evidence gathered in this way can only formally disprove the null hypothesis that the synthetic data deployment is privacy-safe. However, failure to reject the null hypothesis for small $\epsilon$ is nevertheless quite informative when few concrete disclosures are detected, and the disclosure classes accurately characterize the types of risks that the auditor is practically concerned about. 

\subsection{Related work}
Our work provides a framework for auditing synthetically generated data for privacy violations. We provide a brief overview of existing literature on private synthetic generation and privacy auditing.

{\bf Privacy-preserving synthetic data generation}
Since its introduction~\cite{DMNS06}, Differential Privacy (DP) has emerged as the gold standard for formal privacy protection. Recent advancements in generating differentially private synthetic data typically follow one of two paradigms: fine-tuning a model with privacy constraints or utilizing a public model with privacy protections applied during inference.

The first approach~\cite{yue2023synthetic,mattern2022differentially,kurakin2023harnessing} involves learning using DP-SGD~\cite{abadi2016deep} on a private dataset, often starting from a language model pretrained on public data. Once trained, this DP-protected model can be prompted to generate an arbitrary volume of synthetic data without further compromising user privacy. 

The second approach avoids the high computational costs of private training~\cite{ghazi2024differentially} by employing models trained exclusively on public data, but where private data is used somewhere during inference. Some methods in this category prompt the model using private data~\cite{amin2024private,tang2024privacy} and sample synthetic data from a perturbed distribution. Others~\cite{lin2023differentially,xie2024differentially,lin2025differentially} utilize private evolution, where the private data is used to ``vote'' (via DP mechanisms) on which model-generated samples to retain during an iterative process, ensuring the model is never directly queried with private inputs. 

Because our audit methodology requires access only to the resulting synthetic data rather than the underlying model it is applicable to all such generation methods, including those without inherent privacy protections (e.g., sampling from an SFT model trained without DP).

{\bf Privacy auditing}
Generative models are known to sometimes expose their training data creating privacy risks~\cite{carlini2019secret,carlini2023quantifying} and copyright risks~\cite{meeus2024copyright,shi2023detecting} (see~\cite{prashanth2024recite} for a survey on memorization, recitation, and data extraction). Consequently, significant research has focused on empirically testing these systems. Central to this field is the study of attacks on models trained with DP-SGD, primarily through membership inference attacks (MIA)~\cite{JE19, ROF21, CYZF20, jagielski2020auditing,carlini2023quantifying,carlini2019secret,steinke2023privacy} and data reconstruction attacks~\cite{Guo22, Balle22}. 

Empirical testing of generative models generally follows two directions: model-based auditing and data-based auditing. Model-based auditing is the more prevalent approach (see~\cite{steinke2023privacy,shokri2017membership,carlini2019secret,carlini2023quantifying,hallinan2025surprising,jagielski2020auditing,tian2024proving}). This approach typically involves inserting ``canaries'' (i.e., artificially crafted examples) into the training set and subsequently querying the model either for completions or intermediate weights to determine if a specific canary was memorized. {\color{black}{Existing DP audits similarly rely on custom-built canaries and loss-based queries~\cite{2023tightauditingDPML}}}. 

{\bf Auditing Data}
{\color{black}{Whereas dataset inference focuses on identifying IP theft via model loss or margin queries~\cite{maini2024llm}, our work evaluates the privacy risk inherent in the released data itself.}}
The more relevant work to our own is the pioneering study by Meeus et al.~\cite{meeus2025canary}, who were the first to investigate auditing synthetic text data directly. They are motivated by the accuracy of the attack surface; synthetic data is often released without the model used to generate it. Given the proliferation of synthetic data generation techniques, we are also motivated by the need for algorithm-agnostic techniques for evaluating the privacy properties of synthetic data. 

The framework of Meeus et al.~\cite{meeus2025canary} assumes an attacker who can inject canaries into the training set and has access only to the generated synthetic data, not the model itself. The authors demonstrate that training simple n-gram models on the synthetic data allows an attacker to identify inserted canaries. This method performs multiple retrainings of the model on partial subsets of the canary data (to obtain so called reference models). Our approach, notably, does not require such retraining reducing the computational cost.

Prior to this, synthetic data audits focused on specific domains such as images~\cite{hilprecht2019reconstruction} or tabular data~\cite{meeus2023achilles}. In contrast, our work focuses on the same general text data setting as~\cite{meeus2025canary}.



{\bf Canary design}
Since many auditing methods rely on canary insertion, selecting the optimal canary is a non-trivial problem that significantly impacts attack performance. This has motivated research into optimized canary design~\cite{2023tightauditingDPML,meeus2025canary,boglioni2025optimizing}. Our work sidesteps this challenge by not requiring the explicit insertion of canaries, though it remains flexible enough to utilize canary data during the audit, as shown in our empirical evaluation.

\textcolor{black}{
{\bf Relationship to Steinke et al. 
\cite{steinke2023privacy}} That work presents a framework for auditing model training algorithms in one run. Although this prior work still imagines constructing canaries and auditing models, the ideas are critical for auditing  synthetic data setting, where the data are the output of one run of a mechanism. Nevertheless, there are key technical distinctions that require re-proving technical results. First,
existing empirical DP audits focus on the symmetric case where each candidate is included with probability $p = 1/2$. 
and are difficult to adapt to natural, unbalanced train/holdout splits of real datasets without discarding data or re-weighting. Second, a single user's record in training can influence multiple extracted features simultaneously. Disentangling this correlation from the underlying sampling mechanism requires re-deriving the concentration bounds under the compound null hypothesis.}
\section{Disclosures}\label{sec:disclosures}

Consider using a data generation algorithm (see~\cite{ponomareva2025dp} for a recent survey) to create a synthetic version of an email dataset (e.g., the Enron dataset~\cite{klimt2004enron}). Since the algorithm is supposed to learn how emails are formatted and what users write in emails, we may observe an output such as the following.

\begin{verbatim}
 From: john_smith@email.com
 To: jane_doe@email.com
 Subject: New phone number
 
 Dear Jane,
 My new phone number is 212-555-1234.
 Yours,
 John Smith
\end{verbatim}

We expect this data to contain strings that resemble personally identifiable information (PII) such as names, emails, phone numbers, etc. However, if the PII belongs uniquely to a user in the source data --- for example, if $\verb|212-555-1234|$ is an actual user's phone number --- then there is much larger cause for concern. In this work, we call such events \emph{disclosures}.

We loosely define a disclosure as any information that appears rarely in the source corpus, but also appears in the synthetic data. While the specification language is generic, we use it to instantiate three classes of disclosure corresponding to threats that appear commonly in the privacy literature:

\begin{enumerate}
\item {\bf PII leakage.} PII is any string that can trace a user's identity (such as a name, address, account number, etc). A PII disclosure occurs if a string is detected as PII in the source corpus, appears rarely in the source corpus, and also appears in the synthetic corpus. 
\item {\bf Regurgitation.} In text data, regurgitation occurs if a sequence of tokens appears rarely in the source corpus but is reproduced in the synthetic corpus. 
\item {\bf Semantic reconstruction.} Semantic reconstruction occurs if the meaning of a record is preserved in the synthetic data, and few records have the same meaning in the source corpus.
\end{enumerate}

These classes are not meant to be complete, but are illustrative of the types of leakage that an auditor might be practically interested in measuring. A key feature of our system is that it provides a formalism for specifying different classes of disclosures as auditors identify new threats. 

\subsection{Phantom Disclosures}

We note that a disclosure captures a relationship between the synthetic dataset and the private dataset that was used to generate the synthetic data. The same relationships can be measured between the synthetic dataset and an in-distribution private dataset that \emph{was not} used to generate the synthetic data. Suppose we have a dataset $D$ that has been randomly partitioned as $D = D_{\text{train}} \cup D_{\text{hold}}$, and where a synthetic dataset $Y$ was generated from $D_{\text{train}}$. We may take the same measurements on $(Y, D_{\text{hold}})$ as we do on $(Y, D_{\text{train}})$. For example, we may check whether the phone number $\verb|212-555-1234|$ appearing in synthetic data is unique to a holdout user. We call these \emph{phantom disclosures}. 

Naively, one might expect that phantom disclosures happen rarely, if at all. On the contrary, our empirical results demonstrate that such disclosures can and do happen quite frequently. Thus, if we are to use disclosures as evidence of privacy leakage, it is important to establish the phantom disclosure rate as a baseline when interpreting the disclosure rate measured on $(Y, D_{\text{train}})$. 

The existence of phantom disclosures  invites discussion as to \emph{why} phantom disclosures occur. We identify three potential causes:

\begin{enumerate}
    \item {\bf Generalization.} The objective of many synthetic generation procedures is to produce a dataset that is distributionally similar to the source data. For example, the procedure might recognize that training examples predominantly use a particular area code, increasing the probability of incidentally reproducing a holdout phone number. 
    \item {\bf Inductive bias.} The synthetic generation procedure may be predisposed to producing certain types of data before ever encountering training examples. For example, the procedure may have a subroutine that produces US phone numbers as a 10-digit sequence. Alternatively, an LLM used to produce phone numbers may understand their structure from pretraining. 
    \item {\bf Model contamination.} If the synthetic data generation procedure makes use of a foundation model, and the source data is publicly available, it is difficult to guarantee that the heldout data was not seen during pretraining.
\end{enumerate}

Critically, the auditor need not determine which of these mechanisms produced a given phantom. Our framework is agnostic to these causes; it uses the phantom disclosure rate to measure excess disclosures attributable to training-set membership. 


\textcolor{black}{
The issue of model contamination is particularly interesting. Consider again the hospital service that wishes to release synthetic records about patients. Suppose they decide to generate these records by taking an open model, fine-tuning it with real patient data, before sampling from it. Data governance requires isolating the differential risk introduced by fine-tuning of the proprietary data. Establishing the baseline risk that the open model would have produced a patient's data --- for example, because it encountered the patient's name on the web --- is therefore critical to interpreting any perceived disclosures within the synthetic data.  
}
\section{Disclosure Framework}
Our framework allows the system designer to flexibly define classes of disclosure that are of interest to the auditor. We describe this in greater detail. Throughout, we will consider data records that belong to some space $\cX$. When considering private input data, we think of $x \in \cX$ as all the data belonging to a particular user. 

\subsection{Data generation}
We assume access to a combined dataset $D = D_{\text{train}} \cup D_{\text{hold}}$. We also assume as input a randomized mechanism $\mathcal{M}: \mathcal{X}^* \to \mathcal{X}^m$ that takes a subset of $\mathcal{X}^{*}$ and outputs $m$ synthetic records in $\cX$. We denote the synthetic dataset using $Y = \mathcal{M}(D_{\text{train}})$. Importantly, we place no assumptions on how $\mathcal{M}$ is implemented. For example, we do not require that $\mathcal{M}$ constructs a generative model for producing text data, nor do we require glass-box access to the internals of $\mathcal{M}$ in order to audit it. 

\subsection{Features and Rareness}

Recall that we loosely defined a disclosure as any information that is rare in the source data that appears in the synthetic data. In order to define a new class of disclosures, we begin by specifying two functions: $\mathbf{extract}$ and $\mathbf{rare}$.

Given a feature space $\mathcal{F}$, $\mathbf{extract} : \mathcal{X}^* \rightarrow \mathcal{F}^*$ is a function that maps a dataset to a collection of features. $\mathbf{rare} : \mathcal{F} \times \mathcal{X}^* \rightarrow \{0,1\}$ is a function that takes a feature and a dataset and determines whether the feature is rare with respect to the dataset. Given a combined dataset $D = D_{\text{train}} \cup D_{\text{hold}}$, we determine the set of features that are rare relative to $D$: $\mathcal{F}_{\text{rare}} = \{f \in \mathbf{extract}(D) : \mathbf{rare}(f, D) = 1\}$.

For example, the disclosure classes introduced in Section~\ref{sec:disclosures} are instantiated as follows. 

{\bf PII leakage.} $\mathbf{extract}$ returns strings from the source data that appear to be personally identifiable information. For our experiments we use the Google Cloud DLP API to detect such strings. $\mathbf{rare}$ takes a hyperparameter $k$ and indicates that a string is rare if it belongs to $k$ users in $D$. We use $\mathcal{F}_k$ to denote $\mathcal{F}_{\text{rare}}$ when features are defined in this way, and note that $\mathcal{F}_1$ are PII that are unique to users. \textcolor{black}{For example, $\mathrm{extract}$ applied to the sample text in Section~\ref{sec:disclosures} returns  \{`john\_smith@email.com', `jane\_doe@email.com', ...\}, while at $k=1$, $\mathrm{rare}(\text{`john\_smith@email.com'}, D) = 1$ if `john\_smith@email.com' appears exactly once in $D$.}

{\bf Regurgitation.} $\mathbf{extract}$ takes hyperparameters $L_{\min}$ and $L_{\max}$ and returns all substrings in the source data with length in $[L_{\min}, L_{\max}]$. For our experiments, we tokenize substrings by whitespace and measure length by number of tokens. $\mathbf{rare}$ is defined as above. 

{\bf Semantic reconstruction.} $\mathbf{extract}$ maps each user's document to an embedding in $\mathbb{R}^d$. For our experiments, we use Gemini Embedding 2 with $d=768$. $\mathbf{rare}$ takes hyperparameters $\kappa$ and $q$. We compute an approximate $\kappa$-nearest neighbors for the embeddings generated this way using cosine similarity. We rank each embedded datapoint by smallest average similarity to its $\kappa$-nearest neighbors. The record with the least average similarity indicates the user whose record is most semantically distinct from other users in $D$. $q \in [0,1]$ specifies what quantile is considered rare. 

\subsection{Disclosure Quantification}\label{sec:disclosure_quant}

We use rare features to characterize whether the synthetic data set $Y$ discloses information about a user $x \in D$. We do this by way of a function $\textbf{score} : \mathcal{F}^* \times \mathcal{X}^* \rightarrow O$. $\textbf{score}$ is applied on the rare features belonging to user $x$ (i.e. $\mathbf{extract}(x) \cap \mathcal{F}_{\text{rare}}$) and the dataset $Y$. \textcolor{black}{We view $x$ as being equivalent to its rare features. Thus, we overload $\mathbf{score}$ so that it defined on the domain $\mathcal{X} \times \mathcal{X}^*$ as $\mathbf{score}(x, Y) := \mathbf{score}(\mathbf{extract}(x) \cap \mathcal{F}_{\text{rare}}, Y)$, where the domain of the right-hand invocation is $\mathcal{F}^* \times \mathcal{X}^*$.}


{\bf PII leakage and regurgitation.} In these settings, the features are discrete. A natural instantiation of $\textbf{score}$ is therefore to count the number of rare features belonging to a user $x$ that appears in the synthetic data. 
\textcolor{black}{For example, if {`john\_smith@email.com’ and ‘jane\_doe@email.com’} are the only two rare features of $x$ relative to $D$ that appear in $Y$, $\textbf{score}(x, Y) = 2$.}

$\mathbf{score}$ does not need to return a scalar. For discrete settings, we also consider the function $A(x) \in \{0,1\}^{|\mathcal{F}_{\text{rare}}|}$ that lists, for each rare feature, whether user $x$ had the feature and it appeared in $Y$, disaggregating which features of $x$ appeared in the synthetic data.

{\bf Semantic reconstruction.} We instantiate $\textbf{score}$ as the similarity in embedding space between $x$ and its $1$-nearest neighbor in $Y$. When applicable, we can also report the similarities between $x$ and it's $\kappa_Y$-nearest neighbors in $Y$. In the semantic reconstruction setting $x$ generates a single feature $f_x$ for each user $x$ (its embedding), and $\textbf{score}(x, Y) = \bot$ when $f_x \not\in \mathcal{F}_{\text{rare}}$.

\section{Audit Pipeline}

\input{kdd_submission/figures/framework}

We now have all the ingredients to describe our audit pipeline (see Figure \ref{fig:pipeline_diagram}). Begin with a dataset $D = (x_1, \dots, x_n) \in \cX^n$. Each record $x_i$ is assigned a $\mathrm{Bernoulli}(p)$ random variable $s_i$, indicating whether the record will be used in training, defining a partition $D = D_{\text{train}} \cup D_{\text{hold}}$.\footnote{This assignment can be conducted either by a system administrator seeking to conduct an audit, or can occur ``naturally'' if $D_{\text{hold}}$ is already set aside (for instance for validation).} Let $\mathcal{F}_{\text{rare}}$ indicate rare features in the entire corpus for the disclosure type being audited. Run $Y = \cM(D_{\text{train}})$. We measure disclosures on both training $\mathcal{S}_{\text{train}} = \{ \textbf{score}(x, Y) \mid x \in D_{\text{train}}\}$ and hold-out $\mathcal{S}_{\text{hold}} = \{ \textbf{score}(x, Y) \mid x \in D_{\text{hold}}\}$.

$S_{\text{train}}$ describes the disclosures present in the synthetic data release and is valuable information to the auditor and any consumers of the audit, who often simply want to understand whether phenomenon like PII leakage, regurgitation and semantic reconstruction are occurring. Similarly, $S_{\text{hold}}$ describes the phantom disclosure rate. Knowledge of the phantom disclosure rate is often, in isolation, also critical information for decision-makers interpreting the audit. Synthetic data that only appears to leak private information may nevertheless be a problem in real deployments. 

The final step of the audit is to render formal statements about privacy leakage. \textcolor{black}{We implement a two-tiered audit, answering the following questions: (1) Do the disclosures $\mathcal{S}_{\mathrm{train}}$, $\mathcal{S}_{\mathrm{hold}}$ help an attacker to launch an effective membership inference attack? (2) Do they statistically rule-out properties of the mechanism that produced them? To that end, the following taxonomy is useful:}

\textcolor{black}{
\textbf{Membership Inference Attack (MIA)} 
Define the following general MIA game: Fix a mechanism $\mathcal{M}$ and distribution $\pi$. A challenger draws dataset $D_{\mathrm{train}} \sim \pi^n$, candidate $x \sim \pi$, and makes a membership decision $l \in \{-1,+1\}$ where $l = +1$ with probability $p$. If $l=+1$, synthetic data $Y \sim \mathcal{M}(D_{\mathrm{train}} \cup \{x\})$. Otherwise, $Y \sim \mathcal{M}(D_{\mathrm{train}})$. The adversary provides a rule $\phi: (Y, D, x, \mathcal{M}, p, \pi) \mapsto \{-1,0,1\}$, aiming to accurately predict $l$ or abstain.}

\textcolor{black}{
\textbf{Differential Privacy (DP) Audit} A mechanism $\mathcal{M}$ is $\epsilon$-DP if for all $\mathcal{D}_{\mathrm{train}} \in \mathcal{X}^*$, $x \in \mathcal{X}$, measureable sets $\mathcal{Y}$, $\log(\mathbb{P}(\mathcal{M}(D_{\mathcal{\mathrm{train}}}) \in \mathcal{Y})/\mathbb{P}(\mathcal{M}(D_{\mathcal{\mathrm{train}}} \cup \{x\}) \in \mathcal{Y})) \leq \epsilon$. A statistical DP audit observes the output of $\mathcal{M}$ and rejects the hypothesis that $\epsilon \leq  \epsilon_*$ (at some significance level). A special case is the \emph{zero-learning audit} where $\epsilon_* = 0$, and $\mathcal{M}$ is independent of its inputs.}

\textcolor{black}{
We show how the simple disclosure measurements derived from synthetic data (e.g. a piece of PII appearing in $Y$ or a record in $Y$ embedding suspiciously close to a record $x$), can be used to instantiate both DP Audits and MIAs.} 

\section{Statistical Audits and Inference Attacks}
\label{sec:attacks}
\textcolor{black}{
In this section, we define our evaluations. The first we call the {\bf Feature Match Test}, which can be instantiated when  $\mathrm{extract}$ is discrete, and $\mathrm{score}$ is $A(x)$, as defined in Section~\ref{sec:disclosure_quant}. This test is a zero-learning and DP-audit on the synthetic data. Next we propose that thresholding any scalar $\mathrm{score}$ function instantiates an MIA. Finally, the predictions of the MIA can be converted to another DP audit, which we call the {\bf User Match Test}.}


\subsection{Feature Match Test (Discrete Features)}
\label{sec:feature-match-test}
Recall when $\mathbf{extract}$ produces discrete features, we can define $\mathcal{F}_k$ as features that appear $k$ times in $D$. Let $D$ consist of $n$ records derived from users. We index these records by $i = 1,\dots,n$. As previously described, suppose that each record is sampled into $D_{\text{train}}$ independently with probability $p$. Let $l = |\cF_k \cap \mathbf{extract}(Y)|$ denote the number of features in $\cF_k$ that appear at least once in the synthetic output. We index these features by $j = 1,\dots,l$. 

\medskip

\noindent{\bf Audit test statistic.} We construct a binary matrix $A\in\{0,1\}^{n\times \ell}$, where each row $i$ corresponds to an input record $x_i$ and each column 
$j$ corresponds to one of the 
$\ell$ output features $f_j$. We set $A_{ij} = A_{f_j}(x_i, Y)$, indicating whether $x_i$ has feature $f_j$. 

By definition of $\cF_k$, each column has exactly $k$ ones:
$\sum_{i=1}^n A_{ij} = k$, for all $j\in[\ell]$. For each record $i$, define
\begin{align}\label{eq:di}
d_i:= \sum_{j=1}^\ell A_{ij}
\end{align}
the number of output features that appear in record $i$. Our test statistic is defined as the total number of feature-record disclosures associated with the training data:
\begin{align}\label{eq:test-statistic}
T = \sum_{i : x_i \in D_{\text{train}}} d_i 
\end{align}

\subsubsection{\bf Zero learning test} Under the zero learning hypothesis, we consider a mechanism $\cM$ that outputs the same distribution of synthetic data regardless of the input data. Intuitively, since the mechanism does not learn or memorize the input features,  their occurrences in the output should not be overly concentrated in records that happen to be in the training set. Hence, large values of $T$ provide evidence against the zero‑learning hypothesis. This can be formalized as follows.
\medskip

\noindent{\bf Decision rule at significance level $\alpha$.} Under the zero learning null hypothesis, the event that a feature from 
$\cF_k$ appears in the synthetic output is independent of whether a record with this feature was present in the training set. Hence, conditional on the observed matrix $A$, each record is still included in $D_{\text{train}}$ with probability $p$, independent of other records. This allows us to characterize the null distribution by writing the test statistic as:
\begin{align}\label{eq:null-dist-zero}
T \stackrel{(d)}{=} \sum_{i=1}^n d_i s_i \,, 
\end{align}
where `$\stackrel{(d)}{=}$' means in distribution, and we recall that $s_i\in\{0,1\}$ are independent Bernoulli random variables with $\Prob(s_i = 1) = p$.

Under the zero learning hypothesis we have 
\[
\E[T]=\sum_{i=1}^n d_i \E[s_i] = p\sum_{i=1}^n d_i = \ell kp\,.
\]
By applying Hoeffding’s inequality, for any $t>0$, 
\begin{align}\label{eq:Hoeffding1}
\Prob(T-\ell kp \ge t)\le \exp\Big(-\frac{2t^2}{\sum_{i=1}^n d_i^2}\Big)
\end{align}
\textcolor{black}{Setting the right-hand side of \cref{eq:Hoeffding1} equal to $\alpha$ and solving for $t$, we obtain the critical value $c_{\alpha}:= \ell kp + t$ guaranteeing $\Prob(T \ge c_\alpha) \le \alpha$:}
\begin{align}\label{eq:zero-critical}
c_\alpha:=\ell kp +\sqrt{\frac{1}{2}\Big(\sum_{i\in[n]} d_i^2\Big)\log(1/\alpha)} 
\end{align}
We then use the following one‑sided test\footnote{We focus on a one‑sided test because, under the alternative, we expect more matches with the training data than with the hold‑out set. Moreover, the privacy risk we want to assess is the disclosure of features from training data. However, a two‑sided version of the test can also be derived straightforwardly.} for the zero learning hypothesis:
\begin{itemize}
\item Compute the observed statistic $T = \sum_{i\in S} d_i$.
\item Reject the zero learning hypothesis if $T\ge c_{\alpha}$.
\end{itemize}
This rule flags potential true disclosures when the observed counts of 
 synthetic features from the training set is unusually large compared to what would be expected from the no-learning mechanism. Conversely, failing to reject the null indicates that the disclosures observed in the training data cannot be distinguished from those produced completely by chance. 

 We recall that the {\bf type I error} of a test is the probability of falsely rejecting the null hypothesis. By~\eqref{eq:Hoeffding1}, the type I error of our audit is controlled at level $\alpha$; that is, under the zero‑learning hypothesis, $\Prob(T\ge c_\alpha)\le \alpha$. This guarantees that any privacy breach flagged by the test occurs with probability at most $\alpha$ when the mechanism behaves according to the null, providing a clear and calibrated notion of statistical risk for the auditor. 
 
The test also instantiates a p-value, $\pval = \Prob(T\ge c_\alpha)$, \textcolor{black}{and by the duality between confidence intervals and hypothesis testing, this  yields confidence intervals for $p$ (conditioned on the observed test statistic) at any specified confidence level, independent of a specific decision rule. This provides a calibrated measure of how incompatible the observed feature matches are with the zero‑learning hypothesis, and thus can be interpreted as privacy metric in its own right.} We provide a more detailed discussion in Appendix~\ref{sec:pvals}. The same is true for the confidence interval, which we formalize next. \textcolor{black}{We present the test for the zero-learning null hypothesis first because it is conceptually simpler, but we emphasize that rejecting this null is non-trivial for a private model: doing so requires sufficient evidence that rare features unique to the training split surface in the synthetic data. A well-designed private generative model captures population-level distributions, and common semantics rather than outliers, and may therefore exhibit no significant trace of such features despite achieving strong learning and generalization.}

 \medskip

\medskip


\noindent{\bf Confidence interval for sampling probability.} Our test statistic allows us to construct a one-sided confidence interval for the sampling probability $p$, i.e., the probability that a record is included in the training set. This is formalized in the following lemma. 

\begin{lemma}\label{lem:CI-zero}
Fix a significance level $\alpha$ and set 
\begin{equation}
    t_\alpha:= \sqrt{\frac{1}{2}\Big(\sum_{i\in[n]}d_i^2\Big) \log(1/\alpha)}, \quad \hat{p} = \frac{T- t_\alpha}{\ell k}\label{eq:p_hat}
\end{equation}
Then, under the zero learning hypothesis, the following holds with probability at least $1-\alpha$:
\begin{align}
p\in \text{CI}(1-\alpha):=[\hat{p}, 1]
\end{align}
\end{lemma}

In our experiments, (\Cref{tab:combined_stats_sft}) we report $\hat{p}$ as the evidence against the null.


\subsubsection{\bf DP-bounded learning test} 
We next consider a weaker hypothesis than zero learning (but still consistent with strong privacy), where we allow the mechanism to do some learning. We develop audit test for the null hypothesis that the data generation mechanism $\cM$ is $\epsilon$-DP (note that the zero learning claim is consistent with $0$-DP.) As we will discuss our audit test also provides
empirical lower bounds on the privacy loss $\epsilon$. 

Our test statistic is same as the one in zero learning test, given by~\eqref{eq:test-statistic}. For a decision rule, we will work with a larger critical value than the one given in~\eqref{eq:zero-critical}.
\medskip
 
\noindent{\bf Decision rule at significance level $\alpha$.} For a desired significance level $\alpha\in (0,1)$, define the critical value
\begin{align}\label{eq:c_ea}
c_{\epsilon,\alpha}:= \ell k \frac{p}{p+(1-p)e^{-\epsilon}}+ \sqrt{\frac{1}{2}\Big(\sum_{i\in[n]} d_i^2\Big)\log(1/\alpha)}
\end{align} 

We then use the following one‑sided test for DP-bounded learning hypothesis:
\begin{itemize}
\item Compute the observed statistic $T = \sum_{i\in S} d_i$.
\item Reject the $\epsilon$ DP-bounded learning hypothesis if $T\ge c_{\epsilon,\alpha}$.
\end{itemize}
In our next theorem, we show that the type I error of our audit test (its type I error) is controlled at the desired significance.
\begin{theorem}\label{thm:DP-valid}
The type I error of the feature match audit test for the DP-bounded learning hypothesis is controlled at level $\alpha$, i.e., under the null hypothesis we have
\[
\Prob(T\ge c_{\epsilon,\alpha}) \le \alpha\,.
\]
\end{theorem}
The proof of theorem is deferred to the appendix. The proof uses the definition of $\epsilon$-DP to show that under the null hypothesis, and conditional on the synthetic features, the test statistics is stochastically bounded from above and below by two weighted sums of independent Bernoulli random variables. We then use these bounds to prove that for the specific choice of critical value $c_{\epsilon, \alpha}$ the type I error is controlled under the desired level $\alpha$.

\medskip

\noindent{\bf Empirical lower bound on $\epsilon$.} Our analysis of the audit test statistic under $\epsilon$ DP-bounded learning hypothesis can be used to derive empirical lower bound on the privacy loss $\epsilon$. 

\begin{theorem}\label{thm:eps-LB}
For significance level $\alpha\in(0,1)$ and the observed test statistic $T$ given by~\eqref{eq:test-statistic},  define
\[
 A := \frac{T}{\|d\|_1}- \sqrt{\frac{1}{2}\log(1/\alpha)}\frac{\|d\|_2}{\|d\|_1}
\]
\[
B = \frac{T}{\|d\|_1}+ \sqrt{\frac{1}{2}\log(1/\alpha)}\frac{\|d\|_2}{\|d\|_1}
\]
with $d= (d_1,\dotsc,d_n)$ and $d_i$ given by~\eqref{eq:di}.
We define
\begin{align*}
\epsilon_*':=\begin{cases}
\log\left(\frac{1-p}{p}\right) + \log\left(\frac{A}{1-A}\right), & \text{ if } A>0\\
0 & \text{ otherwise}
\end{cases}
\end{align*}
\begin{align*}
\epsilon_*'':=\begin{cases}
\log \left(\frac{p}{1-p}\right) + \log\left(\frac{1-B}{B}\right), & \text{ if } B<1\\
0 & \text{ otherwise}
\end{cases}
\end{align*}
We then have $\epsilon\ge \epsilon_* = \max(\epsilon'_*,\epsilon''_*)$, with probability at least $1-\alpha$.
\end{theorem}
In the case $A\le 0$, the bound $\epsilon'_*$ becomes vacuous, and our procedure does not reject the no‑learning hypothesis (since $T< c_\alpha$). In this regime, the data are therefore consistent with the stronger no‑learning guarantee, which in particular implies that the weaker DP‑bounded learning hypothesis is not violated for any $\epsilon>0$. 

The conclusion of Theorem~\ref{thm:eps-LB} can also be phrased in testing terms: consider the null hypothesis that the mechanism is  $\epsilon$-DP versus its alternative. The smallest value of $\epsilon$ for which our test does not reject this null is $\epsilon_*$.

{\color{black}
\begin{remark}[\textbf{Extension to Approximate $(\epsilon, \delta)$-DP}]
\label{rem:approx_dp}
While Theorem~\ref{thm:eps-LB} is formulated for pure $\epsilon$-DP, our framework directly extends to approximate $(\epsilon, \delta)$-DP. Under $(\epsilon, \delta)$-DP, the expected training match concentration $\mathbb{E}\left[\frac{T}{\ell k} \;\middle|\; Y\right]$ (bounded by $\frac{p}{p + (1-p)e^{-\epsilon}}$ under pure $\epsilon$-DP) is relaxed to $\frac{e^\epsilon + \delta}{1 + e^\epsilon}$, which inverts to:
\begin{equation}
\label{eq:approx_dp_mapping}
\epsilon_*(\delta) = \ln\left( \max\left\{1, \; e^{\epsilon_*} - \delta\left(1 + e^{\epsilon_*}\right)\right\} \right)
\end{equation}
Applying a first-order Taylor expansion to~\eqref{eq:approx_dp_mapping}, we obtain:
\begin{equation}
\epsilon_*(\delta) = \epsilon_* + \ln\left(1 - \delta(1 + e^{-\epsilon_*})\right) \approx \epsilon_* - \delta\left(1 + e^{-\epsilon_*}\right)
\end{equation}
Since $(1 + e^{-\epsilon_*}) \in [1, 2]$ for all $\epsilon_* \ge 0$, the discrepancy between pure and approximate DP bounds is bounded by $|\epsilon_*(\delta) - \epsilon_*| \le 2\delta$. 
Standard DP practices require $\delta \ll 1/N$, where $N$ is the dataset size. Consequently, for practical $\delta \le 10^{-4}$, the shift $2\delta$ is negligible.
\end{remark}

}

\begin{algorithm}[t]
\DontPrintSemicolon %
\caption{User-match audit}
\small
\label{alg:score_audit}
\SetKwInOut{Input}{Input}
\SetKwInOut{Output}{Output}
\Input{
    A dataset $D = \{x_1, \dots, x_{n}\}$, \\
    generative mechanism $\cM$, \\
    privacy budget $\epsilon$ to audit, \\
    Score function $\mathbf{score}:\cX\times\cX^*\to\mathbb{R}$,\\
    classifier $\Phi: \mathbb{R}\to \{-1,0,1\}$ that maps scores $\mathbf{score}(x,Y)$ to predictions, \\
    significance level $\alpha$.
}
\Output{ 1 if the $\epsilon$-DP guarantee is rejected at level $\alpha$,  0 otherwise.}
\BlankLine
Sample $s \in \{0, 1 \}^{n}$ and partition $D$ into $D_{\text{train}} = \{x_i\in D : s_i=1 \}$ and $D_{\text{hold}} = \{x_i\in D : s_i=0 \}$. \;
 $Y \leftarrow \mathcal{M}(D_{\text{train}})$  \tcp*[r]{Generate $m$ synthetic records}
$W_{+1}\leftarrow \{ i\in[n]: \Phi(\mathbf{score}(x_i, Y))=1\}$  \tcp*[r]{Set of positive predictions}
$w_{\text{train}} \leftarrow |\{i \in W_{+1} \text{ and } x_i \in D_{\text{train}}\}|$. \tcp*[r]{Compute true positives}
$r \leftarrow |W_{+1} |$  \tcp*[r]{Compute total number of positive predictions}
\Return{$\mathds{1}(w_{\text{train}} \ge \tilde{c}_{\epsilon,\alpha})$}, where $\tilde{c}_{\epsilon,\alpha}$ is the $(1-\alpha)$-quantile of Binomial$\left(r, \frac{p}{p+(1-p)e^{-\epsilon}}\right)$.
\end{algorithm}

\subsection{User Inference Attack} 
\label{user-inference-attack}

The feature-match test of Section~\ref{sec:feature-match-test} operates
at the level of aggregate statistics: it counts how many rare-feature matches
fall in the training set overall, but does not examine how those matches are
distributed across individual records. A natural question is whether 
disclosure signals can also be used to make predictions about \emph{individual}
users --- that is, whether our framework yields a membership inference attack
(MIA). We show that it does, as a corollary of the auditing pipeline.
Unlike prior MIA methods that require model
access~\cite{carlini2022membership,shokri2017membership,ye2022enhanced},
shadow-model training~\cite{meeus2025canary}, or canary
insertion~\cite{meeus2025canary,2023tightauditingDPML,boglioni2025optimizing},
the resulting attack operates solely on the released synthetic data and the
auditor's held-out sample. 



Let $\textbf{score}: \mathcal{X} \times \cX^* \to \mathbb{R}$ be a scalar scoring function quantifying the similarity between a record $x$ and a synthetic dataset $Y$. We model an adversary who uses these scores to predict whether a record belongs to the training set ($1$) or hold-out set ($-1$), with the option to abstain ($0$) for low-confidence records. The attack is formalized as a classifier $\Phi(\mathbf{score}(x,Y)): \mathbb{R}\to \{-1,+1,0\}$, typically obtained via thresholding on the similarity score $\mathbf{score}(x,Y)$.
\textcolor{black}{Using PII as an example, an attacker might infer that $x$ is in the training data if $x$ contains the PII `john\_smith@email.com', $Y$ contains the PII `john\_smith@email.com', and `john\_smith@email.com' is rare in $D$. This is equivalent to thresholding the scalar PII score function defined in Section~\ref{sec:disclosure_quant} at $1$.}


 These tests are known in the literature as membership inference attacks \cite{steinke2023privacy, meeus2024copyright}. Unlike similar approaches that require baseline normalization, our proposed scores are computed directly from the corpus $D$ and synthetic data $Y$. 
 
 For an $\epsilon$-DP algorithm, the adversary's prediction accuracy is bounded as a function of $\epsilon$. We present the full user-based test in Algorithm~\ref{alg:score_audit} and prove in Theorem~\ref{thm:mia_main} that its type I error is controlled under the target significance level $\alpha$. 

 Note that, as in \Cref{sec:feature-match-test}, we can use the estimated binomial parameter and corresponding adversarial advantage to quantify disclosure risk. Alternatively, given access to intermediate similarity scores $\mathbf{score}(x,Y)$ for records $x \in D_{\text{train}} \cup D_{\text{hold}}$, one can compare with other membership inference attacks using the AUC metric.

\begin{theorem}\label{thm:mia_main}
The type I error of the User-match audit test for the DP-bounded learning hypothesis is controlled at level $\alpha$, i.e., under the $\epsilon$ DP-bounded null hypothesis we have
\[
\Prob(w_{\text{train}}\ge \tilde{c}_{\epsilon,\alpha}) \le \alpha\,.
\]
where $\tilde{c}_{\epsilon,\alpha}$ is the $(1-\alpha)$-quantile of $\text{Binomial}\left(r, \frac{p}{p+(1-p)e^{\epsilon}}\right)$,  $r$ denotes the total number of positive predictions, and $w_{\text{train}}$ is the total number of true positive predictions, as described in Algorithm~\ref{alg:score_audit}.
\end{theorem}
\begin{remark}
By Hoeffding's inequality,
\[
\tilde{c}_{\epsilon,\alpha}
\le
r\frac{p}{p+(1-p)e^{-\epsilon}}
+
\sqrt{\frac{r}{2}\log(1/\alpha)}.
\]
Using this upper bound for the critical value $\tilde{c}_{\epsilon,\alpha}$ still controls type I error at significance level $\alpha$.
\end{remark}
{\color{black} \Cref{rem:approx_dp} also directly applies to the User-Match audit (\Cref{thm:mia_main}), where the concentration corresponds to the proportion of positive predictions originating from the training set ($w_{\text{train}} / r$ in Algorithm~\ref{alg:score_audit}). }

\input{arXiv_submission/experiments_arxiv_2}
\section{Conclusion}
In this work, we introduced a causal auditing framework designed to distinguish between ``true disclosures''---where the synthetic data contains private user information---and ``phantom disclosures'' arising from generalization or coincidence. Our approach provides a rigorous statistical mechanism to distinguish between these two cases and to audit the data against strong privacy bounds---such as zero-learning or specific Differential Privacy (DP) bounds.

Our empirical evaluation confirms that while non-private synthetic data generation processes exhibit significant privacy leakage, disclosures observed in models trained with DP-SGD are statistically indistinguishable from phantom disclosures, validating the efficacy of formal privacy protections. Furthermore, our method offers an interpretable, efficient, and configurable framework to audit against a variety of privacy risks without the need to design canaries or train multiple reference models.

{\bf Limitations and future work.} Several directions remain open. First, our framework currently targets the setting where a data owner controls the train/holdout partition.
Second, while our empirical lower bounds on $\varepsilon$ are valid for the chosen feature set, tighter bounds may be achievable with richer features or more powerful tests. Finally, systematically evaluating disclosure rates across different base models, model scales, and generation paradigms would strengthen the empirical understanding of when and why true disclosures arise.
\clearpage
\bibliographystyle{plain}
\bibliography{refs}

\newpage
\appendix

\input{arXiv_submission/experiment_details_arxiv_2}

\end{document}

%% file: kdd_submission/figures/framework.tex
\begin{figure}[htbp]
    \centering
    \resizebox{0.8\columnwidth}{!}{
    \begin{tikzpicture}[
            node distance=1.0cm and 1.2cm,
            box/.style={
                rectangle, 
                draw=black, 
                thick,
                fill=white, 
                drop shadow,
                minimum width=2.6cm, 
                minimum height=1cm, 
                align=center, 
                font=\sffamily\bfseries\large
            },
            line/.style={
                -{Latex[scale=1.2]}, 
                thick, 
                draw=black!80
            },
            matchline/.style={
                -{Latex[scale=1.2]}, 
                thick, 
                dashed,
                draw=red!70!black
            }
        ]
        
            \node[box, fill=blue!10] (train) {Training Data};
            \node[box, fill=orange!10, right=of train] (model) {Algorithm};
            \node[box, fill=green!10, right=of model] (synth) {Synthetic Data};
        
            \node[box, fill=blue!10, below=of train] (holdout) {Hold-out Data};
        
            \node[box, fill=green!5, below=2.5cm of synth] (feat_synth) {Synthetic\\Features};
            \node[box, fill=blue!5, below=2.5cm of holdout] (feat_holdout) {Hold-out\\Features};
            \node[box, fill=blue!5, left=0.8cm of feat_holdout] (feat_train) {Training\\Features};
        
            \draw[line] (train) -- (model);
            \draw[line] (model) -- (synth);
            
            \draw[line, rounded corners=10pt] (train.north) -- ++(0,0.5) -- ++(-2.4,0) |- (feat_train.north);

            \draw[line] (holdout.south) -- (feat_holdout.north);
            \draw[line] (synth.south) -- (feat_synth.north);
        
            \draw[matchline] (feat_synth.south) to [bend left=25] 
                node [midway, above=11pt, font=\sffamily\small\bfseries, text=red!70!black] {Phantom Disclosures} 
                (feat_holdout.south);
        
            \draw[matchline] (feat_synth.south) to [bend left=35] 
                node [midway, below=10pt, font=\sffamily\small\bfseries, text=red!70!black] {True Disclosures} 
                (feat_train.south);

            \draw[line] (train) edge[dashed,-] node[right, font=\small]{Distributed Like} (holdout);
        
    \end{tikzpicture}
    }
    \caption{{\bf Overview of the audit pipeline.} Data consists of training  and held-out sets. Training data is  processed by a synthetic generation algorithm to produce synthetic data. All three datasets -- training, held-out, and synthetic -- are processed for features. Rare features appearing in held-out and synthetic data are phantom disclosures. Rare features in training and synthetic are true disclosures.}  
    \label{fig:pipeline_diagram}
\end{figure}

%% file: arXiv_submission/experiments_arxiv_2.tex
\section{Experiments}

We evaluate the practical utility and statistical power of the proposed tests through extensive experiments on publicly-available data. We show that both feature-match and user-match tests effectively detect significant disclosures in non-private synthetic data generation processes. In the context of DP synthetic data, we observe a non-negligible number of disclosures. However, we show these are primarily phantoms; our test outcomes and the estimated $\varepsilon$ values confirm the theoretical privacy guarantees of the generation methods used, and the impact of explicit privacy protections in mitigating the risk of true disclosures. 

We benchmark our approach against state-of-the-art synthetic data audit methods by Meeus et al. \cite{meeus2025canary}. To the best of our knowledge this is only data-based audit for synthetic text data.  We show that we obtain tighter disclosure measures for non-private synthetic data without requiring canary insertion or reference model training, and comparable results for differentially private synthetic data.

The remainder of this section is organized as follows. First, we outline the experimental setup: datasets, synthetic data generation processes, feature extraction techniques, and baselines. We then present auditing results for each generation method, comparing our performance against the baseline proposed by \cite{meeus2025canary}.

\subsection{Experimental setup}
{\bf Synthetic Data Generation.}
For each dataset, we partition the records into training and holdout splits by independently selecting each record with probability $p=0.5$. We generate synthetic data through three commonly used techniques:

\begin{itemize}
    \item \textsl{Rewrites}: We prompt \texttt{gemini-3-flash-preview}\footnote{\url{https://ai.google.dev/gemini-api/docs/models/gemini-3-flash-preview}} accessed through its API, to rewrite each record in the corpus, instructing the model to avoid reproducing exact sentences and to refrain from including PII. While this method does not provide theoretical guarantees in terms of DP, it is indeed a common approach in practice~\cite{albanese2026anonymous}.
    \item \textsl{Fine-tuning (SFT):} For fine-tuning we use the open-source Gemma model (\texttt{gemma-3-1b-it} checkpoint)~\cite{team2024gemma} on a next-token prediction task using the training split. All records have an initial prompt of the form ``Write a [content-type]:'', where \textit{[content-type]} corresponds to the content category of the record. The model is fine-tuned using the standard Adam optimizer without privacy constraints.
    \item \textsl{DP-SGD:} Similar to SFT but the model is fine-tuned using a Differentially Private SGD optimizer. We configure the privacy budget to satisfy $(\varepsilon, \delta)$-DP with $\varepsilon=10$ and $\delta=N^{-1.1}$, where $N$ is the training set size.
\end{itemize}

For rewrites, we generate one synthetic record per input record. For SFT and DP-SGD, we generate 100,000 synthetic records from each model by sampling outputs conditioned on the same prompt distribution as the training data.

{\bf Auditing Methods.}
We evaluate our proposed methods and a state-of-the-art baseline:
\begin{itemize}
    \item \textsl{Feature-Match and User-Match (Ours)} The feature-match (\Cref{sec:feature-match-test}) and user-match tests (\Cref{user-inference-attack}), which operate directly on the synthetic output without requiring reference models or canaries. The user-match test has two variants: a verbatim variant based on exact-match features (substring or PII matching), and a variant based on semantic similarity captured via embedding cosine similarity.
    \item \textsl{Canary-based Audit (Meeus et al.):} We implement as baseline the methodology of \cite{meeus2025canary}, which relies on inserting ``hybrid canaries'' into the training data. Following their procedures, we generate 1000 canaries  by sampling real records ($r$), truncating them  after 20 tokens ($r_{prefix}$), and generating a 30-token suffix ($r_{suffix}$) using a pre-trained checkpoint such that the suffix perplexity is contained in the interval $[0.9P, 1.1P]$ (with target perplexity $P=31$). We insert 500 canaries into the training set and keep 500 as holdout. Membership inference scores are computed using an n-gram model and normalized against four reference models trained on disjoint subsets of the canaries.
\end{itemize}

{\bf Feature Extraction.}
To generate the feature set $\cF_k$, we use three different extraction methods:
\begin{itemize}
    \item \textsl{Strings (n-grams):} We extract all contiguous (overlapping) word sequences with lengths in the interval $[n_{min}, n_{max}]$, to capture verbatim disclosures of arbitrary text. For the experiments we focus on sequences with lengths between $n_{min}=11$ and $n_{max}=20$. 
    \item \textsl{DLP (Semantic PII):} We identify semantically sensitive features using the Google Cloud Data Loss Prevention (DLP) tool~\cite{DLP}, accessed via its API. This tool detects a wide variety of PII types, including  identification numbers, financial information, and contact information.
    \item \textsl{Embeddings:} We use the Gemini Embedding method (via API access) to embed texts in 768 dimensions. For each text embedding we calculate its average cosine distance across its 10 nearest neighbors. We then filter out the 0.05 percentile of embeddings with largest average cosine distance for all experiments. 
\end{itemize}
{\bf On the role of feature selection.} The choice of feature extractor is necessarily domain-specific: n-gram matching is appropriate for detecting verbatim memorization in text, DLP-based PII detection targets semantically sensitive information, while embeddings capture semantic similarity. This flexibility allows auditors to test for precisely the types of disclosures they care about, and to obtain \emph{explanations} of detected leakage (e.g., ``the model leaked phone numbers from the training set''). In contrast, model-based MIAs produce a single membership score per record without explaining what was leaked. We compare all feature types across all datasets and show that qualitative conclusions --- rejection of the zero-learning null for non-DP methods (SFT and rewrites), failure to reject for DP-SGD --- are robust to the choice of extractor (see Table~\ref{tab:rewrites_stats}).

We now discuss the results of audits for each generation method. 

\subsection{Rewrites} 
We start by auditing rewrites for six public datasets: Finance~\cite{synthetic-pii-finance-multilingual-2024}, New York Times (NYT) comments\footnote{\url{https://www.kaggle.com/datasets/aashita/nyt-comments}}), Panorama, and Panorama+~\cite{selvam2025panorama}, Postings~\cite{arsh_koneru_2024}, and Tweets\footnote{\url{https://www.kaggle.com/datasets/kazanova/sentiment140/data}}. \Cref{app:datasets} provides further information on each dataset. The Finance and Panorama and Panorama+ datasets are specifically designed to contain a large number of fictitious PII (e.g., names, phone numbers, bank account details). The Tweets and NYT datasets consist of diverse public internet user posts, which sometimes contain personal information linked to their authors. 

First we present the results of a two-sample test using the distributions of cosine similarities between training and holdout records to their nearest synthetic records. Specifically, we compare vectors $\{ x_i\}_i \subseteq \mathbb{R}^d$ and $\{ y_i\}_i\subseteq \mathbb{R}^d$ where $x_{ij}$ is the cosine similarity between training record $i$ and its $j$-th nearest synthetic record, and $y_{ij}$ is analogously defined for holdout records. We set the rarity quantile to $0.05$.

The two-sample test shows that synthetic records are indeed more similar to training than to holdout records: we are able to reject the null hypothesis of zero learning in all datasets. 

\begin{table*}[ht]
    \caption{Feature Match statistics for \textsl{Rewrites} ($k=1$). We report the estimated binomial parameter lower bound $\hat{p}$, highlighting in bold the results for which the null hypothesis ($p\leq 0.5$) is rejected at $\alpha=0.05$.}
    \label{tab:rewrites_stats}
\small
    \centering
    \begin{tabular}{l c r r r r r}
        \toprule
    \multirow{2}{*}{Dataset} & Feature & \multirow{2}{*}{$|\cF_1|$} & \multicolumn{3}{c}{Disclosures} & \multirow{2}{*}{$\hat{p}$} \\
    \cmidrule(lr){4-6}
         & extractor &   & Total & Train & Holdout & \\
        \midrule
Finance      & \multirow{6}{*}{DLP}       & $1.5\times10^4$ & 31     & 17     & 14    & 0.315 \\
        NYT comments &                              & 35              & --     & --     & --    & -- \\
        Panorama     &                              & $3.6\times10^4$ & 39     & 37     & 2     & \textbf{0.743} \\
        Panorama+    &                              & $2.3\times10^4$ & 2      & 1      & 1     & 0.000 \\
        Postings     &                              & $3.3\times10^4$ & 121    & 116    & 5     & \textbf{0.836} \\
        Tweets       &                              & 68              & --     & --     & --    & -- \\
        \midrule
        Finance      & \multirow{6}{*}{Strings} & $5.0\times10^7$ & 28,278 & 26,602 & 1,676 & \textbf{0.792} \\
        NYT comments &                              & $6.9\times10^7$ & 1,368  & 1,358  & 10    & 0.308  \\
        Panorama     &                              & $7.6\times10^7$ & 8,362  & 8,293  & 69    & \textbf{0.548} \\
        Panorama+    &                              & $4.3\times10^6$ & 43     & 23     & 20    & 0.131 \\
        Postings     &                              & $7.7\times10^8$ & 8,245  & 4,492  & 3,753 & 0.456 \\
        Tweets       &                              & $3.1\times10^6$ & 128    & 127    & 1     & \textbf{0.533} \\
        \bottomrule
    \end{tabular}
\end{table*}


We now switch to feature-match and user-match tests to interpret to measure the amount of leakage in terms of the lower bound for the privacy budget $\varepsilon$ using strings, DLP, and embedding-based audits.

{\bf Raw disclosure counts.} We first analyze raw disclosure counts to highlight the importance of distinguishing between true disclosures and phantoms. In Table~\ref{tab:rewrites_stats}, we report results for unique features held by exactly one user ($k=1$) in the data set. (We explore the impact of varying the rarity parameter $k$ in \Cref{sec:app:impact-of-k}.) All hypothesis testing results use the standard significance level $\alpha = 5\%$. Initial inspection of \Cref{tab:rewrites_stats} reveals a high volume of rare features in the original data ($|\cF_1|$), ranging between 35 and $7.7\times10^8$ instances. We then use our framework to establish if these matches are evidence of actual leaks.


{\bf Feature-match audit.} We focus first on feature matches using DLP. Unsurprisingly given that we prompted explicitly the model to rewrite the data removing the PII, we find that many synthetic rewrite datasets have few or no PII disclosures. Nevertheless, as our audit shows, these disclosures can still happen and be significant privacy violation in some cases. For example, in the synthetic Postings dataset, we detect that out of the $121$ total PII matches found, $116$ are from the training set, while only $5$ appear in the holdout set. 
The overwhelming excess of training matches over holdout matches provides strong evidence of leakage as confirmed by the binomial lower bound in \Cref{eq:p_hat}: $\hat{p} = 0.84$, rejecting the zero-learning hypothesis. This example highlights both that synthetic rewriting via LLMs is not always successful in removing privacy risk and that our framework can detect such remaining risk. 
Similar leakage is detected in the Panorama dataset. 

Similarly, using rare strings, we find significant leakage for Finance, Panorama, and Tweets for exact verbatim matches with the training data than the holdout, showing that even when instructing a model to avoid repetition, the model still replicates significant parts of input records exactly. 
Nevertheless, for some datasets we do not find evidence of leakage of verbatim substrings or PII (NYT comments, Tweets, or Panorama+)  using feature-match tests alone. 

{\bf User-match audit.} We next evaluate the performance of the User-Match test. To ensure a direct comparison with the state-of-the-art baseline by \cite{meeus2025canary}, we report the Area Under the Curve (AUC) as the evaluation metric, following their protocol. To evaluate the significance of the results we use the standard non-parametric Mann–Whitney–Wilcoxon test~\cite{mcknight2010mann} (a.k.a.\ Mann–Whitney U test) for AUC (with significance level $5\%$). 

\Cref{tab:rewrites_auc} summarizes the results.\footnote{Mann–Whitney–Wilcoxon test for AUC computes the p-value based on the deviation of AUC from 0.5, accounting for the number of positive and negative cases. When evaluating significant results, note that AUC alone is not a measure of statistical significance; a high AUC may not be significant in small sample sizes, while a modest AUC can be significant with large sample sizes.} We observe that our test based on embeddings (semantic similarity) outperforms the \textsl{Meeus et al.}\cite{meeus2025canary} baseline on all datasets. Moreover, both tests (embedding- and string-based) reveal that NYT comments and Tweets rewrites leak enough information to re-identify individual training records with significant accuracy: the AUC using semantic features is 0.93 for NYT comments and 0.86 for Tweets. To illustrate the risk, \Cref{tab:semantic-matches} shows some example training records paired with their closest synthetic matches, identified using our embedding-based approach. We observe often that substantial rephrasing is present, so these leaks would go  unnoticed by standard n-gram based tests, including \cite{meeus2025canary}. (More examples are reported in \Cref{sec:app:extended-table-embedding}.) 

\begin{table}[h]
    \caption{User Inference Attack AUC scores for \textsl{Rewrites}. Our \textsl{User-Match} test (specifically using embeddings) outperforms the canary-echo baseline. Asterisks indicate significant results using the Mann–Whitney–Wilcoxon test for AUC. Bold font indicates the best method in terms of AUC.}
    \label{tab:rewrites_auc}
    \centering
    \small
    \begin{tabular}{lccc}
    \toprule
    Dataset & Strings & Embeddings & Canary-Echo \\
     & (Ours) & (Ours) & Meeus et. al \cite{meeus2025canary} \\
    \midrule
      Finance      &  ${\bf 0.65}^\star$ & $0.53^\star$ & $0.53^\star$ \\
      NYT Comments & $0.73^\star$ & ${\bf 0.93}^\star$ & $0.51^\star$    \\
      Panorama     & ${\bf 0.73}^\star$ & $0.52^\star$          & $0.51^\star$    \\
      Panorama+    & 0.48          & ${\bf 0.65}^\star$             & $0.52^\star$ \\
      Postings     & 0.49          & ${\bf 0.75}^\star$ & $0.52^\star$    \\
      Tweets       & 0.68 & ${\bf 0.86}^\star$ & $0.51^\star$ \\
    \bottomrule
    \end{tabular}
\end{table}

{\bf DP-bounded tests.} As introduced in \cref{sec:feature-match-test}, we map the statistics from feature-match and user-match tests to differential privacy budgets $\varepsilon$ via the DP-bounded learning tests. We present here the user-match $\varepsilon$ values (derived via \Cref{thm:mia_main}) and defer to \Cref{tab:epsilon_estimates_feature_match} in \Cref{sec:app:epsilon_estimates_feature_match} the values derived from the feature-match $\hat{p}$ estimates in \Cref{tab:rewrites_stats} (which also map to positive $\varepsilon$ estimates when $\hat{p}>0.5$).{ \color{black} While reported for $\delta=0$, by Remark~\ref{rem:approx_dp} these lower bounds hold up to values $\pm 10^{-4}$ for any realistic approximate DP parameter $\delta \le 10^{-4}$. In particular, in our benchmark, the smallest corpus is Panorama+ ($N \approx 4{,}000$ records, thus $\delta$ must hold $\delta \le 10^{-4}$)).
}

\Cref{tab:epsilon_estimates} reports the results. The estimated $\varepsilon$ values reach as high as 5.0 for Tweets and 5.3 for NYT, indicating substantial privacy leakage for these datasets under the rewriting procedure. Notice how our method using either embeddings or strings always outperforms the Meeus et al. baseline with embeddings being particularly effective.

\begin{table}[h]
 \caption{\textcolor{black}{User-match derived $\varepsilon$ lower bounds for the DP-bounded test on \textsl{Rewrites}. Reported values hold within a numerical precision of $\pm 10^{-4}$ for any approximate DP parameter $\delta \le 10^{-4}$.}}
    \label{tab:epsilon_estimates}
\small
    \centering
    \begin{tabular}{ l r r r r }
        \toprule
    Dataset  & \multicolumn{3}{c}{$\mathbf{\epsilon} $ (Ours)} & $\mathbf{\epsilon}$ (Meeus et al.\cite{meeus2025canary})  \\
    \cline{2-4}
          & DLP & Strings & Embeddings \\
        \midrule
        Finance  &  0.0 & 2.04 & 0.25 & 1.77 \\
        NYT & 0.0 &  0.65 & 5.28 & 0.09 \\
        Panorama  &   1.65 & 0.81 & 0.99 & 0.26\\
        Panorama+ &  0.0 & 0.0 & 0.80 & 0.39\\
        Postings & 2.30 & 0.02 & 2.73 & 0.24 \\
        Tweets & 0.0 & 1.29 & 5.01 & 0.16 \\
    \bottomrule
    \end{tabular}
\end{table}

\begin{table*}[htbp]
    \centering
    \caption{Examples of semantic and structural leakage from rewrites synthetic datasets. Training records are paired with their closest synthetic match using cosine similarity. Long texts are truncated for brevity. Across all datasets, rewrites leak semantic information from training data, which is captured by cosine distance in the embedding space.}
    \label{tab:semantic-matches}
    \renewcommand{\arraystretch}{1.4} 
    \begin{tabularx}{\textwidth}{@{} l X X @{}}
        \toprule
        \textbf{Dataset} & \textbf{Original Train Record (Cropped)} & \textbf{Closest Synthetic Record (Cropped)} \\
        \midrule
        \textbf{NYT} &
        ... think that \match{FIFA is going to void the championships} ... owned by \match{cocaine} landlords such as the \match{Rodriguez Orejuela}... \match{America, Nacional and Millonarios} have not been punished... &
        ... \match{FIFA will ever revoke the league trophies} ... subsidiaries of the Cali and Medellín cartels... \match{Rodriguez Orejuela} used their \match{drug} billions... \match{America de Cali, Atletico Nacional, and Millonarios} have kept their stars... \\

        \addlinespace
        \textbf{Panorama} &
        \match{Richard Hunt} (born July 19, 1954...) is a dedicated \match{childcare worker} known for his lifelong commitment... \match{born to} Lance and \match{Gabrielle Hunt}... &
        \match{Richard Hunt} is a dedicated Canadian \match{childcare worker}... \match{born} on October 11, 1946... to Jeremy and \match{Laura Hunt}... \\

        \bottomrule
    \end{tabularx}
\end{table*}

{\bf Interpretability and Efficiency.} Beyond these numerical comparisons, our approach offers two key advantages over existing methods. First, the baseline evaluates privacy risk exclusively on inserted, out-of-distribution canaries, making it difficult to translate those metrics into realistic privacy risks for true records. In contrast, our User-Match test measures the disclosure of \textit{actual} user data from the training corpus, yielding a more actionable and interpretable assessment of a model's privacy leakage regarding its training subjects. 
Second, our method captures semantic memorization that evades traditional verbatim filters. Even when a model reproduces training data, it often does so non-consecutively, scattering or reordering the leaked information throughout the generated output. Because standard $n$-gram or strict substring metrics require contiguous text overlap, they systematically fail to capture this dispersed leakage (as shown in \cref{tab:semantic-matches}).

Finally, our method is significantly more efficient, requiring only a single round of training and sampling. In contrast, the \cite{meeus2025canary} procedure is computationally expensive, requiring the training of four separate reference models (shadow models) on disjoint subsets of canaries (in addition to the primary training corpus) to normalize membership scores of the audited synthetic data. Finally, we observe that our method does not require the insertion of canaries, simplifying its deployment. Nevertheless, is possible to optionally incorporate the  canaries into our own auditing framework. Doing so  yields similar results (see \Cref{sec:app:canary-match-tests}).

\subsection{SFT and DP generated synthetic data}
\label{sec:sft-dp-sft}
We next audit two additional procedures used in the literature to generate synthetic datasets, namely supervised fine-tuning (SFT) and differentially private SFT through DP-SGD,  focusing on the three datasets with explicit fictitious PII insertion: Finance, Panorama, and Panorama+.\footnote{Due to computational costs of generating the synthetic data, we limit ourselves to these datasets.} 

As with rewrites, we start with a two-sample test using the embedding similarity distributions for training and holdout records. For the rarity threshold of $0.05$ we observe that this test rejects the null hypothesis for 2 of the three datasets: Finance, Panorama. (More results are reported in~\Cref{sec:app:extended-table-embedding}). 

{\bf Feature-match tests.} We continue with the feature-match test in \Cref{tab:combined_stats_sft}. For the \textsl{SFT} models, we reject the null hypothesis of the zero-learning test ($p \leq 0.5$) with high confidence in {\it all} datasets for at least one feature type (DLP or strings). The estimated lower bounds $\hat{p}$ (bolded in \Cref{tab:combined_stats_sft}) significantly exceed $0.5$. 
In contrast, for the \textsl{DP-SGD} generation processes, all $\hat{p}$ values are close to or less than 0.5. In all cases, we fail to reject the zero-learning null hypothesis. This again confirms the validity of our tests and the importance of accounting for phantom disclosures: while matches do occur even in DP data (e.g., 15 matches in Finance using DLP), they are distributed evenly between train and holdout, consistent with the random baseline. \Cref{tab:epsilon_estimates_sft} reports the corresponding $\varepsilon$ lower bounds for SFT; all DP-SGD estimates are zero.

\begin{table}[t!]
    \caption{Feature Match statistics using DLP and strings extractors on synthetic data generated from \textsl{SFT} and \textsl{DP-SGD}, for $k=1$. We show the total number of disclosures appearing in either the train or holdout splits, and the counts of each split. We present the estimated binomial parameter lower bound $\hat{p}$, and present in bold results where the null ($p\leq 0.5$) 
    is rejected at significance $\alpha=0.05$.}
    \label{tab:combined_stats_sft}
\small
    \centering
\begin{tabular}{l c r r r r}
        \toprule
    \multirow{2}{*}{Dataset} & Feature & \multicolumn{3}{c}{Disclosures} & \multirow{2}{*}{$\hat{p}$} \\
    \cmidrule(lr){3-5}
         & extractor & Total & Train & Holdout & \\
        \midrule
        \multicolumn{6}{c}{\textbf{Method: SFT}} \\
        \midrule
        Finance  & \multirow{3}{*}{DLP} & 763 & 492 & 271 & \textbf{0.590} \\
        Panorama  &                        & 55  & 39  & 16  & \textbf{0.544} \\
        Panorama+ &                        & 586 & 356 & 230 & \textbf{0.556} \\
        \midrule
        Finance   & \multirow{3}{*}{strings} & 20,491 & 11,942 & 8,549 & \textbf{0.827}\\
        Panorama  &                           & 5,114  & 3,211  & 1903 & \textbf{0.654}\\
        Panorama+ &                           & 664   & 322   & 342  & 0.390\\
        \midrule
        \multicolumn{6}{c}{\textbf{Method: DP-SGD}} \\
        \midrule
        Finance   & \multirow{3}{*}{DLP} & 26 & 15 & 11 & 0.302\\
        Panorama+ &                      & 0  & 0  & 0  & 0.0\\
        Panorama  &                      & 3 & 1 & 2 & 0.0\\
        \midrule
        Finance   & \multirow{3}{*}{\shortstack{strings}} & $13,693$ & 7,269 & 6,424 & 0.458\\
        Panorama+ &                                                                 & 0        & 0    & 0    & 0.0\\
        Panorama  &                                                                 & 73       & 38   & 35   & 0.302\\
        \bottomrule
    \end{tabular}
    \end{table}

\begin{table*}[h]
    \caption{User Inference Attack AUC scores for string features. We compare our \textsl{User-Match} test against the \cite{meeus2025canary} baseline. We observe that \cite{meeus2025canary} finds significant results only on the Panorama+ datasets and that our method outperforms the baseline on the Finance and Panorama datasets with \textsl{Non-DP} generated data while requiring 5$\times$ fewer computational resources for training and sampling. Dashes (-) indicate cases where insufficient matches were found to calculate an AUC. Asterisks indicate significant results using the Mann–Whitney–Wilcoxon test for AUC. Bold font indicates the best method in terms of AUC.}
    \label{tab:audit-user-sft}
    \centering
    \begin{tabular}{llccc}
    \toprule
    Data & Dataset & Substring features & Embeddings features & Canary-Echo \\
    Generation & & (Ours) & (Ours) &  Meeus et al. \cite{meeus2025canary} \\
    \midrule
    \multirow{3}{*}{\textbf{SFT}} 
      & Finance   & ${\bf 0.63}^\star$ & 0.53 & 0.48 \\
      & Panorama  & ${\bf 0.53}^\star$ & 0.52 & 0.50 \\
      & Panorama+ & 0.48          & 0.50 & ${\bf 0.56}^\star$ \\
    \midrule
    \multirow{3}{*}{\textbf{DP-SGD}} 
      & Finance   & {\bf 0.52} & 0.50 & 0.49 \\
      & Panorama  & {\bf 0.50} & 0.46 & {\bf 0.50} \\
      & Panorama+ & -    & - & ${\bf 0.56}^\star$ \\
    \bottomrule
    \end{tabular}
\end{table*}

{\bf User-match tests.} We evaluate the User-Match test on SFT and DP-SGD generated data. \Cref{tab:audit-user-sft} summarizes the results. 
For \textsl{SFT}, at least one of our methods yields significant AUC for Finance and Panorama; results on \textsl{DP-SGD} data are not significant. Both of our proposed user-match tests outperform the \textsl{Meeus et al.} baseline on  two of the three datasets, while \textsl{Meeus et al.} obtains significant results only on Panorama+ (both SFT and DP-SGD). Specifically, on the Finance dataset, we achieve an AUC of 0.63, significantly higher than the baseline's 0.48. On Panorama, we achieve 0.53 compared to 0.50. 

\begin{table}[h]
 \caption{\textcolor{black}{User-match derived $\epsilon$ lower bounds for the DP-bounded test on SFT-generated data. All \textsl{DP-SGD} estimates are not significant ($\epsilon = 0.0$). Reported values hold within $\pm 10^{-4}$ precision for any $\delta \le 10^{-4}$.}}
    \label{tab:epsilon_estimates_sft}
\small
    \centering
    \begin{tabular}{ l r r r }
        \toprule
    Dataset  & \multicolumn{3}{c}{$\mathbf{\epsilon}$} \\
    \cline{2-4}
          & DLP & Strings & Embeddings \\
        \midrule
        \multicolumn{4}{c}{\textbf{Method: SFT}} \\
        \midrule
        Finance  & 0.65 & 3.24 & 0.24 \\
        Panorama &  1.13 & 1.16 & 0.97\\
        Panorama+  &  0.0 & 0.34 & 0.2\\
    \bottomrule
    \end{tabular}
\end{table}

%% file: arXiv_submission/experiment_details_arxiv_2.tex
\section{Experimental details}
\label{sec:app-experiments}
\subsection{Datasets}
\label{app:datasets}

We leverage three existing publicly available datasets for synthetic data generation. 
{\bf Finance} This dataset contains 55,940 fictitious financial records generated using the Gretel package. It contains financial documents with distinct PII types in different languages. For details on the generation, document and PII subcategories, and software, see \cite{synthetic-pii-finance-multilingual-2024}. This dataset contains 82 unique prompts specifying the type of document (e.g. a statement, a formal document, insurance claim) and the type of sensitive features to include (e.g. patient information and diagnosis or income). One example is ``A form capturing personal, financial, and employment details of a loan applicant, including loan amount, purpose, and supporting documents.''

{ \bf NYT comments} This dataset contains comments on articles published in the New York Times during the preiod of January through May 2017, and January through April 2018. We sample 100,000 records.

{ \bf Panorama} The Panorama dataset introduced by \cite{selvam2025panorama} contains 384,789 records from 9674 fictitious user profiles emulating online activity by the users (e.g., social media posts). The records contains personally identifiable information from the (fictitious) authors of the content designed to evaluate PII detection methods and LLM auditing tools. We accessed the dataset on December 1, 2025. The dataset contains six different prompts of the form ``write a [content-type]'' where content-type is a field specified for each record and corresponds to one of the categories wiki-style articles, social media posts, forum discussions, online reviews, comments, and marketplace listings. 

{\bf Panorama+} The Panorama Plus dataset, also introduced by \cite{selvam2025panorama}, contains 9,674 structured user profiles that serve as the ground truth identities for the Panorama dataset. This dataset has the explicit PII attributes for each fictitious user, including full names, family details, socio-economic status, health attributes, and other sensitive demographics. We accessed the dataset on December 1, 2025. 

{\bf Postings} This is a dataset~\cite{arsh_koneru_2024} containing job postings from LinkedIn listed in 2023 and 2024.

{\bf Tweets} This dataset contains tweets extracted using the twitter API. We sample 100,000 records.

\subsection{Impact of Feature Rarity ($k$)}
\label{sec:app:impact-of-k}

While the main text focuses on unique features ($k=1$), we extend our analysis here to features shared by multiple users. \Cref{fig:impact-of-k} plots the estimated lower bound of the binomial parameter, $\hat{p}$, as a function of the rarity parameter $k$.

We observe that $|\cF_k|$ (the number of features appearing in the synthetic output that are shared by \textit{exactly} $k$ users in the original corpus) decreases as $k$ increases. This reduction in the effective sample size leads to looser estimates of $\hat{p}$. As the number of available matches diminishes, the statistical confidence decreases, making it increasingly hard to distinguish true disclosures from phantoms. 

To illustrate this, \Cref{fig:impact-of-k-confidence} presents the point estimate for the binomial parameter (computed as $T^{obs}/\ell k$). The shaded region indicates the confidence interval.  While all tests are significant for $k=1$,  the lower bound diverges from the point estimate as feature counts increases. Larger values of $k$ are primarily relevant for datasets where users share secrets, resulting in correlated or shared features. In contrast, the datasets studied here are synthetic and generated via independent sampling for each user.
\begin{figure}
    \centering
    \includegraphics[width=\linewidth]{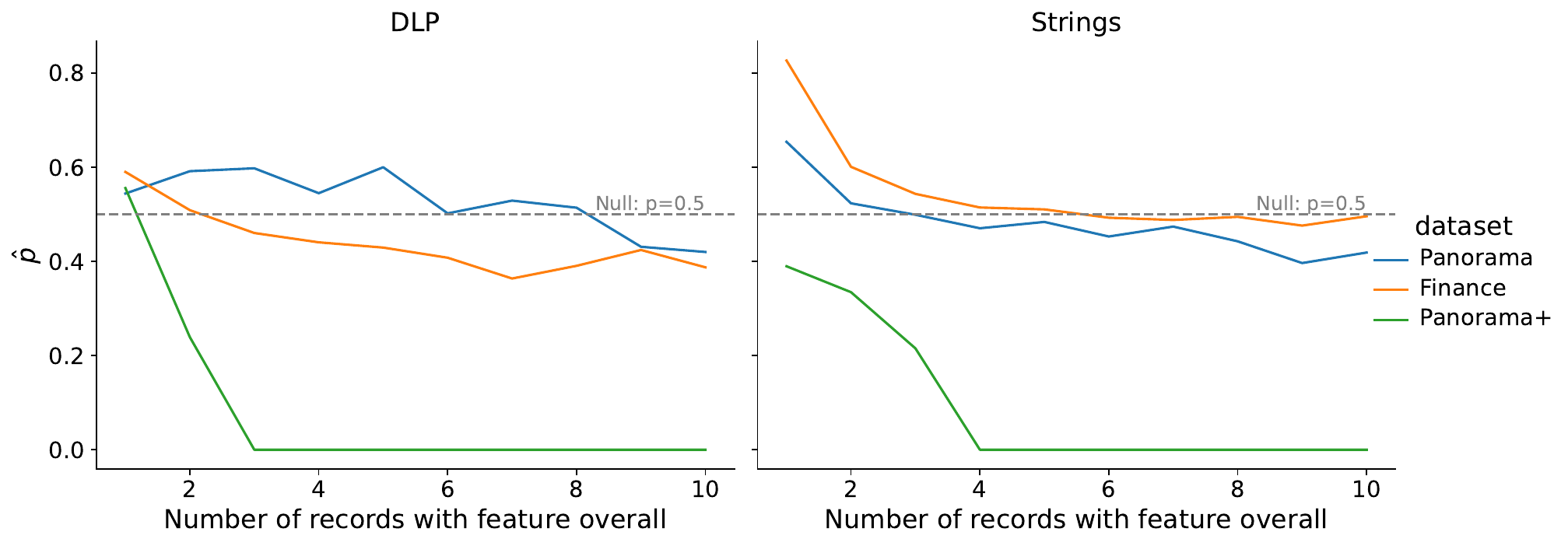}
    \caption{\textbf{Impact of feature rarity ($k$) on the audit lower bound $\hat{p}$.} As $k$ increases (features shared by more records), the number of valid audit features $|\cF_k|$ decreases. This reduction in sample size leads to looser statistical bounds, reducing the power of the
    tests.}
    \label{fig:impact-of-k-confidence}
\end{figure}

\begin{figure}
    \centering
    \includegraphics[width=0.9\linewidth]{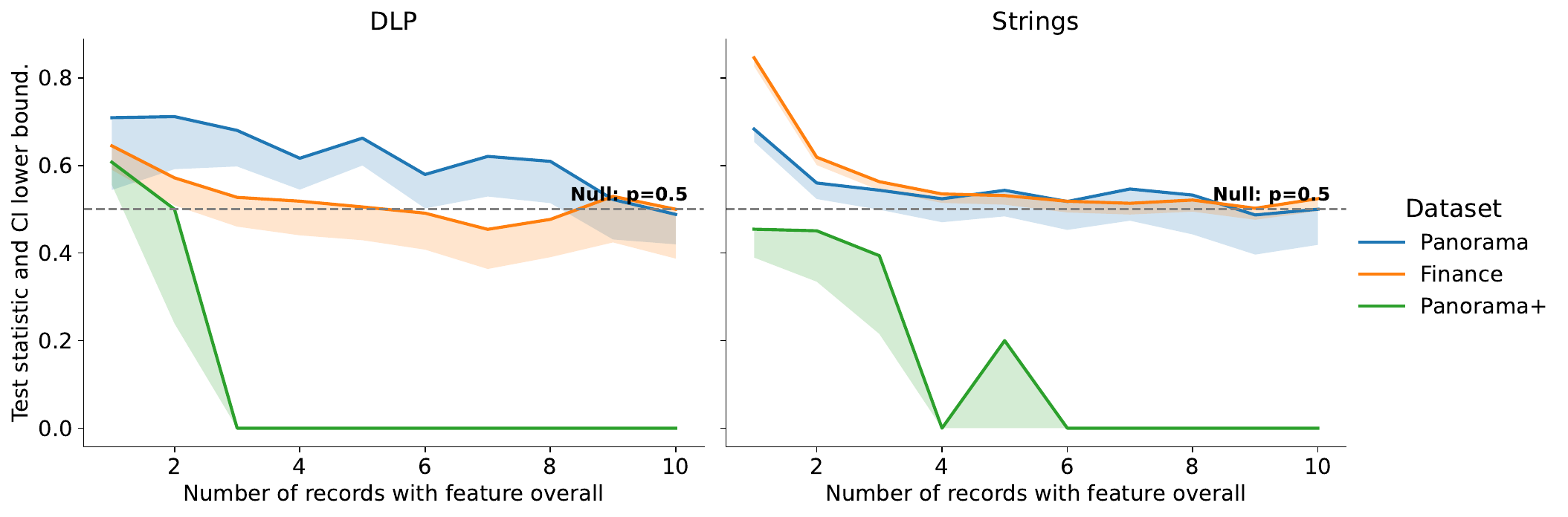}
    \caption{\textbf{Point estimates and confidence intervals for varying $k$.} The solid line represents the empirical match rate (point estimate). The shaded region illustrates the gap between the point estimate and the conservative lower bound $\hat{p}$ used for auditing. The gap highlights the increased statistical uncertainty caused by the lack of non-unique features in the synthetic corpus at higher frequencies. 
    }
    \label{fig:impact-of-k}
\end{figure}

\subsection{Impact of rarity threshold on the two-sample test}\label{sec:app:rarity-two-sample}

\begin{figure}
    \centering
    \includegraphics[width=0.45\textwidth]{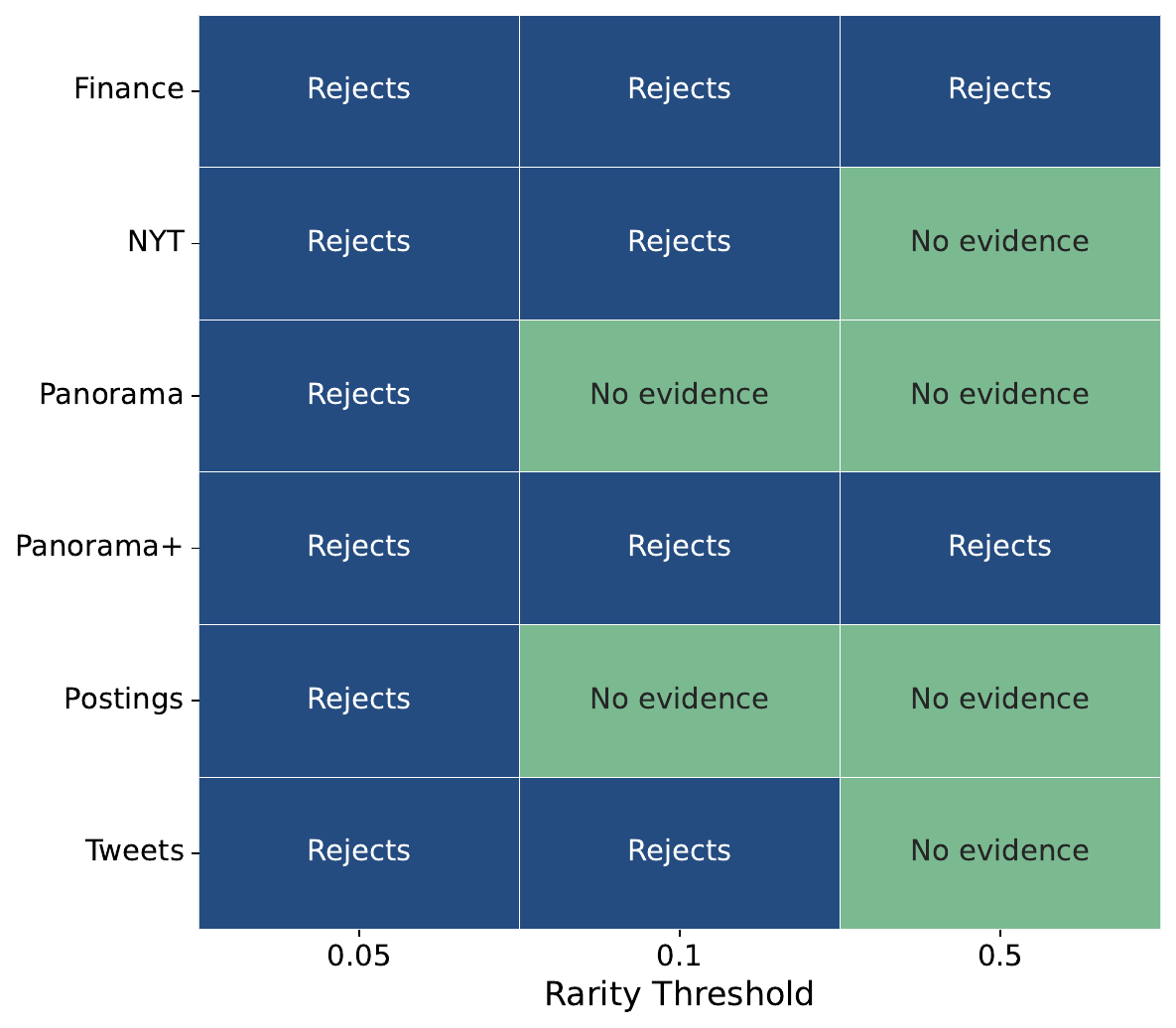}
    \caption{Two-sample test results comparing the distribution of cosine similarities between training records and synthetic rewrites against the distribution of cosine similarities between holdout records and synthetic rewrites. As we include more common records, the power of the test decreases for most datasets.}
    \label{fig:cosine_dist_comparison}
\end{figure}

In \Cref{fig:cosine_dist_comparison} we report results for the effect of the rarity threshold on the two-sample test on the rewrite data. Notice that in all datasets, the rarity threshold $0.05$ is significant. As we include more common records in the test (i.e., as the rarity threshold increases), the distributions become to distinguish. Nevertheless we observe that the test is able to identify significant leakage in all datasets for low thresholds. 

In \Cref{fig:sft_embeddings} we report results for SFT. Unlike rewrites, as we increase the rarity threshold, the test can becomes more powerful in some cases. This is consistent observations of \cite{meeus2025canary}, fine-tuned models are more likely to generate common patterns than rare unique features. Thus, by considering more rare embeddings, it is less likely to find matches in the synthetic data, as these rare patterns may not have been learned by the fine-tuned model yet or may require a larger number of synthetic samples to appear.

\begin{figure}
    \centering 
    \includegraphics[width=0.45\textwidth]{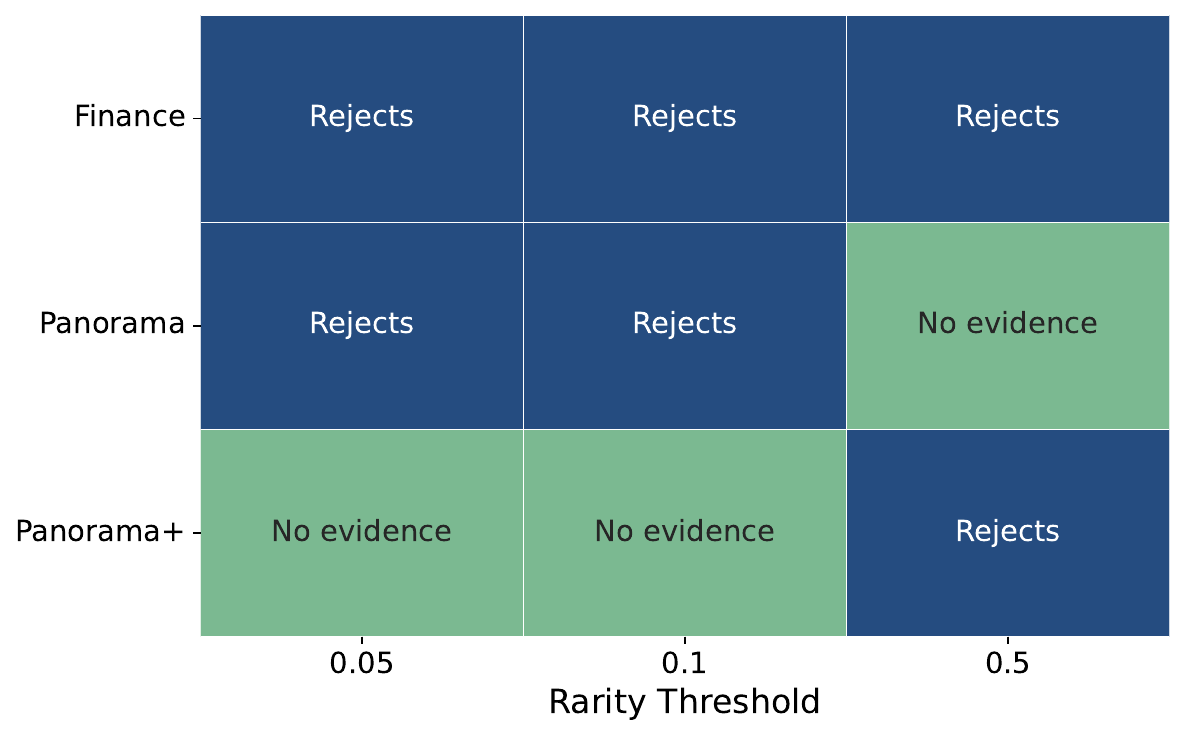}
    \caption{Two-sample test comparing distributions of training-to-synthetic and holdout-to-synthetic embedding similarities for SFT generated data}.
    \label{fig:sft_embeddings}
\end{figure}

\subsection{Example matches identified by the embedding-based method}\label{sec:app:extended-table-embedding}

We report in \Cref{tab:semantic-matches-app} more examples of matches identified by the embedding-based method for datata obtained by rewrite. Notice the ability of the method to identify leakage even in the presence of extensive rewriting.

\begin{table*}[htbp]
\small
    \centering
    \caption{Examples of semantic and structural leakage from rewrites synthetic datasets. Training records are paired with their closest synthetic match using cosine similarity. Long texts are truncated for brevity. Across all datasets, rewrites leak semantic information from training data, which is captured by cosine distance in the embedding space.}
    \label{tab:semantic-matches-app}
    \renewcommand{\arraystretch}{1.4} 
    \begin{tabularx}{\textwidth}{@{} l X X @{}}
        \toprule
        \textbf{Dataset} & \textbf{Original Train Record (Cropped)} & \textbf{Closest Synthetic Record (Cropped)} \\
        \midrule
        \textbf{NYT} &
        ... think that \match{FIFA is going to void the championships} ... owned by \match{cocaine} landlords such as the \match{Rodriguez Orejuela}... \match{America, Nacional and Millonarios} have not been punished... &
        ... \match{FIFA will ever revoke the league trophies} ... subsidiaries of the Cali and Medellín cartels... \match{Rodriguez Orejuela} used their \match{drug} billions... \match{America de Cali, Atletico Nacional, and Millonarios} have kept their stars... \\

        \addlinespace
        \textbf{Panorama} &
        \match{Richard Hunt} (born July 19, 1954...) is a dedicated \match{childcare worker} known for his lifelong commitment... \match{born to} Lance and \match{Gabrielle Hunt}... &
        \match{Richard Hunt} is a dedicated Canadian \match{childcare worker}... \match{born} on October 11, 1946... to Jeremy and \match{Laura Hunt}... \\

        \addlinespace
        \textbf{Finance} &
        09:00:01 - Job Ferran: Goedemorgen, bedankt voor de hulp. \match{Ik heb moeite met het aanmaken van mijn account.} ... &
        10:01:22 Customer: Goedemorgen/nedeem/middag, \match{ik heb moeite met het aanmaken van mijn account.} ... \\

        \addlinespace
        \textbf{Finance} &
        \match{MsgType=35} \match{SenderCompID=AcmeBroker TargetCompID=Exchange} BeginString=FIX.5.0 News=\match{Market volatility} expected... &
        ... \match{MessageType = 35} \match{SenderCompID = AcmeTrading TargetCompID = Exchange} NewsType = 35 \match{NewsText = "Acme Trading} announces significant \match{market volatility}"... \\
                \addlinespace
        \textbf{Panorama+} &  [...] \match{children\_count: 0 credit\_score: 521 disability: None drivers\_license: GB-DL-}6687495077 \match{employer: Group Store fathers\_name:} Antony \match{Simmons} finance\_status: Low first\_name: Georgia \match{gender: Female job\_title: Retail Sales Associate last\_name: Simmons locale: en\_GB marital\_status: Single mothers\_name: }Shirley \match{Simmons national\_id: GB-ID}-9261993 \match{nationality: British} net\_worth: £8342.85 [...] \match{spouse\_name: N/A }         &  \match{children\_count: 0 credit\_score: 612 disability: None drivers\_license: GB-DL-}71839502649 [...] \match{employer: Retail Group Ltd fathers\_name: }Mark \match{Simmons} finance\_status: Medium first\_name: Chloe \match{gender: Female job\_title: Sales Assistant last\_name: Simmons locale: en\_GB marital\_status: Single mothers\_name: }Sarah \match{Simmons national\_id: GB-ID}-8829401 \match{nationality: British} net\_worth: £-1150.40 \match{spouse\_name: N/A} [...] \\
         \textbf{Tweets} & \match{According to the} 2009 Cedar Point Rider Safety guide, \match{I am 30 lbs too heavy, 3 inches too tall, and chest measurement that is too large} & \match{Based on the} 2015 Alton Towers Accessibility and Health Manual, \match{my current stats put me 14 kilograms over the limit, 8 centimeters past the height ceiling, and possessed of a shoulder span that simply won't clear the harness lock.} \\
        \bottomrule
    \end{tabularx}
\end{table*}

\subsection{DP-bounded learning tests}
\label{sec:app:epsilon_estimates_feature_match}
We present the results for DP-bounded learning tests on \textsl{Non-DP} generated data. The $\hat{p}$ estimated values map to positive $\varepsilon$ estimates, confirming significant privacy violation. We present these results in \Cref{tab:epsilon_estimates_feature_match}. All $\epsilon$ estimates for \textsl{DP-SGD} generated data are not significant. This means we fail to reject the hypothesis of positive $\epsilon$  values based on the synthetic data. This confirms the positive effect of using Differential Privacy in preventing true disclosures.

\begin{table}[h]
 \caption{Feature Match statistics for a DP-bounded test, using DLP and strings extractors on synthetic data generated for \textsl{Non-DP}.  We present the estimated $\epsilon$ privacy budget lower bound, and present in bold results where the null is rejected at significance $\alpha=0.05$. All \textsl{DP-SGD} estimates are not significant.}
    \label{tab:epsilon_estimates_feature_match}
\small
    \centering
    \begin{tabular}{ l c r }
        \toprule
    \multirow{2}{*}{Dataset} & Feature  & \multirow{2}{*}{$\epsilon$} \\
         & extractor & \\
        \midrule
        \multicolumn{3}{c}{\textbf{Method: SFT}} \\
        \midrule
        Finance  & \multirow{3}{*}{DLP} & \textbf{0.364} \\
        Panorama+ &                       &  \textbf{0.223} \\
        Panorama  &                       &  \textbf{0.177} \\
        \midrule
        Finance   & \multirow{3}{*}{Strings}  & \textbf{1.561}\\
        Panorama+ &                           & 0.0 \\
        Panorama  &                           & \textbf{0.636}\\   
        \bottomrule
    \end{tabular}
\end{table}
\subsection{Feature and user match audits with canaries}
Our framework can easily be adapted to work with canaries by simply inserting the train canaries to the train set $D_{train}$, and holdout canaries to the holdout set $D_{hold}$. 
We present these results in \Cref{tab:appendix-audit-user}. Adding canaries yields mixed results, providing a slight performance boost on Panorama (0.533 $\to$ 0.549) but showing negligible impact on Finance. In the DP settings, most scores lie near the random baseline of 0.5 or fail to produce sufficient matches (indicated by ``-''), consistent with the privacy guarantees of the model. We highlight that, while increasing MIA scores is desirable and adding canaries can help with this, it does not necessarily provide a causal explanation for natural data disclosures, which is the focus of the present work. However, exploring how to combine these methods is an interesting avenue of future work.

\label{sec:app:canary-match-tests}
\begin{table}[h]
\caption{User Match AUC scores for string features ($n \in [11, 20]$). We compare our \textsl{User-Match} test (with and without canaries) against the \textsl{Canary-Echo} baseline. Our framework outperforms the baseline on the Finance and Panorama datasets while requiring 5$\times$ fewer computational resources for training and sampling. Dashes (-) indicate cases where insufficient matches were found to calculate an AUC. Asterisks indicate significant results using the Mann–Whitney–Wilcoxon test for AUC. Bold font indicates the best method in terms of AUC.}
    \label{tab:appendix-audit-user}
    \centering
    \begin{tabular}{llccc}
    \toprule
     & & \multicolumn{2}{c}{\textbf{Ours}} & \textbf{Baseline} \\
    \cmidrule(lr){3-4} \cmidrule(lr){5-5}
    Generation & Dataset & User Match & User Match & Canary-Echo \\
    Model & & & (+ Canaries) & Meeus et al.\cite{meeus2025canary} \\
    \midrule
    \multirow{3}{*}{\textbf{Non-DP}} 
      & Finance   & ${\bf 0.63}^\star$ & ${\bf 0.63}^\star$ & 0.47 \\
      & Panorama  & $0.53^\star$ & ${\bf 0.55}^\star$ & 0.50 \\
      & Panorama+ & 0.48 & 0.47 & ${\bf 0.55}^\star$ \\
    \midrule
    \multirow{3}{*}{\textbf{DP-SGD}} 
      & Finance   & {\bf 0.52} & 0.47 & 0.49 \\
      & Panorama  & {\bf 0.50} & 0.38 & {\bf 0.50} \\
      & Panorama+ & -     & -     & ${\bf 0.56}^\star$ \\
    \bottomrule
    \end{tabular}
\end{table}

\section{P-Value Interpretation}\label{sec:pvals}

 \noindent{\bf P-value.} 
Instead of committing to a specific significance level $\alpha$ in advance, the auditor may report a p-value, which summarizes how strongly the data speak against the zero‑learning hypothesis. A concrete type I error guarantee is obtained once a decision threshold $\alpha$ is fixed a priori and the rule ``reject if $\pval\le \alpha$'' is adopted.
 
\begin{lemma}\label{lem:pval-zero}
Define
\begin{align*}
    \pval:=\begin{cases}
    \exp\Big(-\frac{2(T- \ell kp)^2}{\sum_{i=1}^n d_i^2}\Big),  & \text{ if }T\ge \ell kp,\\
    1 & \text{otherwise}
    \end{cases}
\end{align*}
Then, $\pval$ is a valid $p$-value for the zero learning hypothesis, i.e, $\Prob(\pval\le t)\le t$ for any $t\in[0,1]$.
\end{lemma}
\begin{proof}[Proof of Lemma~\ref{lem:pval-zero}] Recalling~\eqref{eq:Hoeffding1} we have the following for any $t>0$,
\begin{align}\label{eq:hoeff-repeat}
\Prob(T-\ell kp \ge t)\le \exp\Big(-\frac{2t^2}{\sum_{i=1}^n d_i^2}\Big)
\end{align}
We consider two cases.

\noindent$\bullet$ {\bf Case 1:} $T< \ell kp$. Then 
$\pval=1 \ge t$ for any 
$t\in[0,1]$, so the event 
$\{\pval\le t\}$ has probability zero.

\noindent$\bullet$ {\bf Case 2:} $T\ge \ell kp$. The event $\{\pval\le t\}$ is equivalent to
$\exp\Big(-\frac{2(T- \ell kp)^2}{\sum_{i=1}^n d_i^2}\Big)\le t \Leftrightarrow-\frac{2(T- \ell kp)^2}{\sum_{i=1}^n d_i^2} \le \log t \Leftrightarrow T- \ell kp \ge \sqrt{\frac{1}{2} (\sum_{i=1}^n d_i^2)\log(1/t)}$
where we used that $T \ge \ell kp$ and $\log t<0$ for $t\in[0,1)$.
Using~\eqref{eq:hoeff-repeat} with this threshold, the last term has probability at most $t$.

Since $\{\pval\le t\}\subseteq \{T\ge \ell kp\}$, we obtain $\Prob(\pval\le t)\le t$.

The two cases together show that 
$\pval$ is a valid p‑value.
\end{proof}

Likewise, we can construct the p-value for testing the DP-bounded learning hypothesis. Let $q_\epsilon:= \frac{p}{p+(1-p)e^{-\epsilon}}$ and define 
\begin{align*}
    \pval:=\begin{cases}
    \exp\Big(-\frac{2(T^{\rm obs}- \ell kq_\epsilon)^2}{\sum_{i=1}^n d_i^2}\Big),  & \text{ if }T^{\rm obs}\ge \ell kq_\epsilon,\\
    1 & \text{otherwise}
    \end{cases}
\end{align*}
Then, $\pval$ is a valid $p$-value for the DP-bounded learning hypothesis, i.e, $\Prob(\pval\le t)\le t$ for any $t\in[0,1]$. The proof of validity follows along the same lines as in the proof of Lemma~\ref{lem:pval-zero} and is omitted.

\section{Proof of theorems and technical lemmas}\label{app:proofs}

\subsection{Proof of Lemma~\ref{lem:CI-zero}}
As discussed in~\eqref{eq:null-dist-zero}, under the zero‑learning hypothesis, the test statistic $T$ has the same distribution as $ \sum_{i=1}^n d_i s_i$, with $s_i\in\{0,1\}$ independent Bernoulli random variables with $\Prob(s_i = 1) = p$. Hence, by applying Hoeffding bound with the threshold $t_\alpha$, we obtain
\begin{align}
\Prob(T-\ell k p> t_\alpha) \le \alpha\,,
\end{align}
or equivalently
\begin{align}
\Prob(T-\ell k p\le t_\alpha) \ge 1-\alpha\,,
\end{align}
The event $\{T-\ell kp\le t_\alpha\}$ is equivalent to the event $\{(T-t_\alpha)/(\ell k)\le p\}$. Hence, with probability at least $1-\alpha$, we have
\[
p\in \Big[\frac{T- t_\alpha}{\ell k}, 1\Big]\,,
\]
which completes the proof.
\subsection{Proof of Theorem~\ref{thm:DP-valid}}
For each record $i\in [n]$, let $s_i\in\{0,1\}$ be the indicator that record $i$ belongs to the training set $D_{\text{train}}$ (denoted by $s_i =1$) or the hold-out set $D_{\text{hold}}$ (denoted by $s_i=0$). For each feature $j\in\cF$, let $o_j\in\{0,1\}$ be the indicator that feature 
$j$ appears in the synthetic output. We write $ o= \cM(s)$ with $o\in\{0,1\}^{|\cF|}$ and $s\in\{0,1\}^n$. 

We first establish upper and lower tail bounds for the distribution of the test statistic under the DP-bounded learning hypothesis. Our proof strategy follows the lines of Proposition 5.1 in \cite{steinke2023privacy}, but generalizes their analysis from the special case $p=1/2$ to arbitrary sampling probabilities $p\in(0,1)$. More importantly, our test statistic—based on feature-unit incidence patterns—requires a  different analysis.

\begin{proposition}\label{propo:feature-k-DP}
Suppose that the mechanism $\cM$ is $\epsilon$-DP. Conditional on the synthetic features in the output $(o = o^*)$, the test statistic from~\eqref{eq:test-statistic} is  stochastically dominated by 
$\sum_{i\in[n]} d_i z_i$, 
where $\{z_i\}_{i=1}^n$ are i.i.d. Bernoulli$(q_\epsilon)$ random variables.
In addition, it stochastically dominates $\sum_{i\in[n]} d_i z'_i$, where $\{z'_i\}_{i=1}^n$ are i.i.d. Bernoulli$(q'_\epsilon)$ random variables, with
\begin{align*}
 q'_\epsilon:= \frac{p}{p+(1-p)e^\epsilon}, \quad q_\epsilon :=  \frac{p}{p+(1-p)e^{-\epsilon}}
\end{align*}
Formally, for any value $v\in \mathbb{R}$, 
$\Prob(\sum_{i=1}^n d_i z'_i \ge v|o=o^*) \le \Prob(T\ge v| o= o^*) \nonumber
\le \Prob(\sum_{i=1}^n d_i z_i \ge v|o=o^*)$
\end{proposition}
Using Proposition~\ref{propo:feature-k-DP}, we have
$
\Prob(T\ge c_{\epsilon,\alpha}|o=o^*) \le 
\Prob(\sum_{i=1}^n d_i z_i \ge c_{\epsilon,\alpha}|o=o^*) = \Prob\Big(\sum_{i=1}^n d_i z_i - q_\epsilon \sum_{i=1}^n d_i \ge \|d\|_2\sqrt{\frac{1}{2} \log(1/\alpha)}\;\Big|\;o=o^*\Big) \le \alpha\,,
$
where the last step follows from the Hoeffding bound applied to the sum  $\sum_{i=1}^n d_i z_i$.
Note that conditional on $o=o^*$, the quantities $d_i$ become deterministic. The result then follows by marginalizing over the output $o$.

\subsubsection{Proof of Proposition~\ref{propo:feature-k-DP}}
We use asterisks $s^*,o^*$ to indicate a realization of vectors $s = (s_1,\dotsc, s_n)$ and $o =(o_1,\dotsc,...,o_{|\cF|})$. 

Fix $i\in[n]$, and $s^*_{<i} \in\{0,1\}^{i-1}$. By
Bayes’ law 
\begin{align}
&\Prob(s_i=1|\cM(s)= o^*, s_{<i} = s^*_{<i})\nonumber\\
&= 
\frac{\Prob(M(s) = o^*|s_i = 1, s_{<i}=s^*_{<i})\; \Prob(s_i=1|s_{<i} = s^*_{<i})}{\Prob(\cM(s) = o^*| s_{<i} = s^*_{<i})}\label{eq:bayes}
\end{align}
We have $\Prob(s_i=1|s_{<i} = s^*_{<i}) = \Prob(s_i=1) = p$ as the records are included in the training set independently. We also have
\begin{align}
&\Prob(\cM(s) = o^*| s_{<i} = s^*_{<i})\nonumber\\
&= \Prob(\cM(s) = o^*| s_i = 1,s_{<i} = s^*_{<i}) \;\Prob(s_i=1)\nonumber\\
&\;\;\;+ \Prob(\cM(s) = o^*| s_i = 0,s_{<i} = s^*_{<i})\; \Prob(s_i=0)\nonumber\\
&= \Prob(\cM(s) = o^*| s_i = 1,s_{<i} = s^*_{<i}) p\nonumber\\
&\;\;\;+ \Prob(\cM(s) = o^*| s_i = 0,s_{<i} = s^*_{<i})(1-p)\label{eq:DP-ratio}
\end{align}
By definition of $\epsilon$-DP we have
\begin{align}
\frac{\Prob(\cM(s) = o^*| s_i = 1,s_{<i} = s^*_{<i})}{\Prob(\cM(s) = o^*| s_i = 0,s_{<i} = s^*_{<i})} \in [e^{-\epsilon}, e^{\epsilon}]\,,
\end{align}
which along with~\eqref{eq:DP-ratio} implies that
\begin{align}
&\frac{\Prob(\cM(s) = o^*| s_{<i} = s^*_{<i})}{\Prob(\cM(s) = o^*| s_i = 1,s_{<i} = s^*_{<i})}\nonumber \\
&\in \left[p+(1-p) e^{-\epsilon}, p+(1-p) e^{\epsilon} \right]
\end{align}
Using this bound in~\eqref{eq:bayes} we obtain
\begin{align}
&\Prob(\cM(s) = o| s_{<i} = s^*_{<i})
\in \left[\frac{1}{1+\frac{1-p}{p}e^\epsilon}, \frac{1}{1+\frac{1-p}{p}e^{-\epsilon}} \right]\,.
\end{align}

 We proceed by proving the claim via induction on the number of records $n$. Note that the test statistics can be written as $\sum_{i\in[n], j\in[|\cF|]} s_i \delta_{ij} o_j$ where $\delta_{ij} = 1$ if feature $j$ appears in record $i$ and $\delta_{ij}=0$ otherwise.
We have
%
\begin{align*}
&\sum_{i\in[n],j\in[|\mathcal{F}|]} s_i \delta_{ij} o_j \\
&= \sum_{i\in[n-1],j\in[|\mathcal{F}|]} s_i \delta_{ij} o_j + s_n \Big( \sum_{j\in[|\mathcal{F}|]} \delta_{nj} o_j \Big) \\
&= \sum_{i\in[n-1],j\in[|\mathcal{F}|]} s_i \delta_{ij} o_j + d_n s_n,
\end{align*}
where the last equality holds since $d_n$ is defined as the number of features in the output that appear in record $n$.

By the induction hypothesis, conditional on $o = o^*$, the partial test statistic $T_{n-1}:=\sum_{i\in[n-1],j\in[|\cF|]} s_i \delta_{ij} o_j$ is stochastically dominated by $\sum_{i\in[n-1]} d_i z_i$. 
 Moreover, as shown above, conditioning on $(o, s_1, \dotsc, s_{n-1})$, the variable $s_n$ is stochastically dominated by $z_n\sim\text{Bernoulli}(q_\epsilon)$. Since $d_n$ is measurable with respect to $o$,  it follows that  $d_n s_n$ is dominated by $d_n z_n$. Lemma~\ref{lemma:u-U} below then implies that, conditional on $(o,T_{n-1})$, $d_n s_n$ is stochastically dominated by $d_n z_n$.  

\begin{lemma}\label{lemma:u-U}
If $X|U$ is stochastically dominated by $Y|U$, then for any deterministic function $f$, $X|f(U)$ is stochastically dominated by $Y|f(U)$.
\end{lemma}

To complete the induction step, we apply \cite[Lemma 9]{steinke2023privacy}, stated here for convenience:
\begin{lemma}\label{lem:dominance}
Suppose $X_1$ is stochastically dominated by $Y_1$. Suppose that, for all $x\in\mathbb{R}$, the
conditional distribution $X_2|X_1 = x$ is stochastically dominated by $Y_2$. Assume that $Y_1$ and
$Y_2$ are independent. Then $X_1 + X_2$ is stochastically dominated by $Y_1 + Y_2$.
\end{lemma}

Using lemma~\ref{lem:dominance} for
\begin{align*}
X_1 &= T_{n-1} (\text{conditional on }o=o^*) \\
Y_1 &= \sum_{i\in[n-1]} d_i z_i (\text{conditional on }o=o^*) \\
X_2 &= d_n s_n (\text{conditional on }o=o^*)\\ 
Y_2 &= d_n z_n (\text{conditional on }o=o^*)    
\end{align*}
 Lemma~\ref{lem:dominance} yields the desired result. A nearly identical argument applies to obtain the claimed lower tail bound using $z'_i$.
 \smallskip
 
\noindent{\bf Proof of Lemma~\ref{lemma:u-U}}
Let $\Omega_b = f^{-1}(b)$. Then $\Prob(X \ge v \mid f(U) = b) = \Prob(X \ge v \mid U \in \Omega_b) = \frac{\Prob(U \in \Omega_b \mid X \ge v) \Prob(X \ge v)}{\Prob(U \in \Omega_b)} = \frac{\int_{\Omega_b} \Prob(U = u \mid X \ge v) \, \de u}{\Prob(U \in \Omega_b)} \Prob(X \ge v)$.
Since $X|U=u$ is dominated by $Y|U=u$, we have $$\Prob(X\ge v|U=u)\le \Prob(Y\ge v|U=u),$$  and thus 
\[
\Prob(U=u|X\ge v) \Prob(X\ge v) \le \Prob(U=u|Y\ge v) \Prob(Y\ge v).
\]
Substituting this bound yields
\begin{align*}
&\Prob(X \ge v \mid f(U) = b)\\
&\le \frac{\int_{\Omega_b} \Prob(U = u \mid Y \ge v) \, \de u}{\Prob(U \in \Omega_b)} \Prob(Y \ge v) \\
&= \Prob(Y \ge v \mid f(U) = b),
\end{align*}
which completes the proof. $\square$

\subsection{Proof of Theorem~\ref{thm:eps-LB}}
By Proposition~\ref{propo:feature-k-DP}, under the DP-bounded learning hypothesis and conditional on the observed output, the test statistic $T$  is stochastically dominated by $\sum_{i=1}^n d_i z_i$ where $z_i\sim\text{Ber}(q_\epsilon)$, independently. Hoeffding's inequality then yields
\[
\Prob\left(\sum_{i=1}^n d_i z_i - q_\epsilon\|d\|_1 \ge t \right)\le \exp\left(-\frac{2t^2}{\|d\|_2^2}\right).
\]
Setting $t = \|d\|_2\sqrt{(1/2)\log(1/\alpha)}$ gives
\[
\Prob\left(\sum_{i=1}^n d_i z_i \ge q_\epsilon\|d\|_1 + \|d\|_2\sqrt{(1/2)\log(1/\alpha)} \right)\le \alpha\,.
\]
Thus, with probability at least $1-\alpha$,
\[
T\le q_\epsilon \|d\|_1 +\|d\|_2\sqrt{(1/2)\log(1/\alpha)}
\]
Rearranging and substituting $q_\epsilon = \frac{p}{p+(1-p)e^{-\epsilon}}$ yields
\[
A:= \frac{T}{\|d\|_1} - \frac{\|d\|_2}{\|d\|_1}\sqrt{\frac{1}{2}\log(1/\alpha)} \le \frac{p}{p+(1-p)e^{-\epsilon}}\,.
\]
Solving for $\epsilon$ gives the lower bound $\epsilon\ge \epsilon_*'$.

Likewise, $T$ stochastically dominates $\sum_{i=1}^n d_i z_i$, where $z_i\sim\text{Ber}(q_\epsilon)$, independently. The one-sided Hoeffding bound gives
\[
\Prob\left(\sum_{i=1}^n d_i z'_i \le q'_\epsilon\|d\|_1 - \|d\|_2\sqrt{(1/2)\log(1/\alpha)} \right)\le \alpha\,,
\]
or equivalently,
\[
\Prob\left(\sum_{i=1}^n d_i z'_i \ge q'_\epsilon\|d\|_1 - \|d\|_2\sqrt{(1/2)\log(1/\alpha)} \right)\ge 1-\alpha\,.
\]
Thus, with probability at least $1-\alpha$, 
\[
T\ge q'_\epsilon\|d\|_1 - \|d\|_2\sqrt{(1/2)\log(1/\alpha)}.
\]
Rearranging the terms yields
\[
B:=\frac{T}{\|d\|_1} + \frac{\|d\|_2}{\|d\|_1}\sqrt{\frac{1}{2}\log(1/\alpha)} \ge q'_\epsilon = \frac{p}{p+(1-p)e^{\epsilon}}\,.
\]
Solving for $\epsilon$ gives the lower bound $\epsilon\ge \epsilon_*''$.

Combining these bounds, under the $\epsilon$-DP hypothesis we have $\epsilon\ge \epsilon_* = \max(\epsilon_*',\epsilon_*'')$, with probability at least $1-\alpha$.
\subsection{Proof of Theorem~\ref{thm:mia_main}}We recall that the vector $s\in\{0,1\}^n$ encodes which records are in the training set or the hold-out set. Let $u\in\{-1,0,1\}^n$ with $u_i:=\Phi(g(x_i,Y))$ denoting the adversary's guess of $s_i$. Under the null hypothesis, $\cM$ is $\epsilon$-DP and so by  the postprocessing property of DP, $u$ is  an $\epsilon$-DP function of $s$.

By definition, $w_{\text{train}} = \sum_{i=1}^n \max\{0,u(s_i) s_i\}$. By a similar argument as in Proposition 5.1 in \cite{steinke2023privacy}, we have 
\begin{align*}
        &\Prob\left( \sum_{i=1}^{n}  \max\{0,u(s_i)\cdot s_i\} \geq v\;\Big|\; u(s)=u \right) \\
        &\leq \Prob\left(\sum_{i=1}^{n} \tilde{S}_i |u_i| \geq v \right) \,,
    \end{align*}
where $\tilde{S}_i \sim \text{Bern}\left( \frac{p}{p+(1-p)e^{-\varepsilon}}\right)$ are independent random variables. 

Since $\Phi$ can ``abstain'' to make predictions for certain scores, let $r(u)$ be the number of predictions given $u=u(s)$. We can average over the zeroed bernoullis to get that 
\begin{align}
\Prob\left(\sum_{i=1}^{n} \tilde{S}_i |u_i| \geq v \right) = \Prob(X\ge v)
\end{align}
where $X\sim \text{Binomial}\left(r,\frac{p}{p+(1-p)e^{-\epsilon}}\right)$.

By combining the previous equations, we get
\[
\Prob(w_{\text{train}}\ge v\;|\; u(s) = u)\le \Prob(X\ge v)\,.
\]
By setting $v=\tilde{c}_{\epsilon, \alpha}$, the $(1-\alpha)$- quantile of $\text{Binomial}\left(r,\frac{p}{p+(1-p)e^{-\epsilon}}\right)$, we get
\[
\Prob(w_{\text{train}}\ge \tilde{c}_{\epsilon,\alpha}\;|\; u(s) = u)\le \alpha\,,
\]
which completes the proof.
{

\textcolor{black}{
\section{Model-Access vs. Data-Only Auditing}
\label{app:model_vs_data}
}
\looseness=-1
We evaluate the relationship between data-only auditing (privacy leakage from the generated synthetic dataset $Y$) and model-access auditing (privacy leakage on the fine-tuned model parameters $\theta$ via white-box/grey-box MIA). 
We consider the setting from \cref{sec:sft-dp-sft} where an auditor evaluates a synthetic data generation pipeline that fine-tunes a model on a private dataset $D_{\text{train}}$ with inserted canaries $S_{\text{in}}$, alongside held-out canaries $S_{\text{out}}$. We benchmark the data audits from \cite{meeus2025canary} and our own data audits (Algorithm~\ref{alg:score_audit}) against the following standard model-based audits:\looseness=-1

\textbf{Raw target loss MIA (\cite{yeom2018privacy, shokri2017membership}}: The auditor evaluates per-token cross-entropy loss directly on the target fine-tuned checkpoint $\mathcal{M}_{\text{target}}$:
    \begin{equation}
        s_{\text{loss}}(x) = - \mathcal{L}(\mathcal{M}_{\text{target}}, x) = \frac{1}{T}\sum_{t=1}^T \log P_{\mathcal{M}_{\text{target}}}(x_t \mid x_{<t})
    \end{equation}
\textbf{Reference-Calibrated LiRA MIA \cite{carlini2022membership}}: The auditor computes likelihood ratios calibrated against $K=4$ reference shadow models $\mathcal{M}_{\text{ref}, 0}, \dots, \mathcal{M}_{\text{ref}, 3}$ trained on independent splits:
    \begin{equation}
        s_{\text{calibrated}}(x) = \frac{1}{K}\sum_{k=0}^{K-1} \mathcal{L}(\mathcal{M}_{\text{ref}, k}, x) - \mathcal{L}(\mathcal{M}_{\text{target}}, x)
    \end{equation}

\textbf{Parameter memorization vs. actual data leakage}:
    In SFT, model-access loss auditing achieves the highest raw ranking separation (e.g., $\text{AUC}=0.62$ on Panorama). The gap between model loss and data audits quantifies the effect of stochastic generation (sampling temperature $T > 0$, top-$p$, and finite output budget $|Y|$), where only a fraction of memorized representations actually become data disclosures.
 When fine-tuning with DP-SGD, empirical leakage across all model-access and data-only modalities reduces to baseline random chance ($\text{AUC} \approx 0.50$, $\hat{\epsilon} \approx 0.00$, and non-significant $p$-values). 

\textbf{Practical need of data auditing}:
    While white-box model auditing provides stronger detection when model weights and shadow reference models are accessible, real-world synthetic data pipelines frequently use proprietary APIs (e.g., Gemini, GPT), or model free methods (e.g. private evolution) where model weights do not exist or cannot be audited. Data-only remains the sole viable tool in these settings. In addition, it can outperform model and existing data-audits without additional model training nor canary insertion. 
}